\documentclass{article}

    \PassOptionsToPackage{}{natbib}

\usepackage[preprint]{neurips_2020}
\usepackage{graphicx}

\usepackage{amsmath,amsfonts,bm}

\def\eqref#1{equation~\ref{#1}}

\def\1{\bm{1}}

\def\eps{{\epsilon}}

\DeclareMathAlphabet{\mathsfit}{\encodingdefault}{\sfdefault}{m}{sl}
\SetMathAlphabet{\mathsfit}{bold}{\encodingdefault}{\sfdefault}{bx}{n}

\DeclareMathOperator*{\argmin}{arg\,min}

\usepackage[utf8]{inputenc} 
\usepackage[T1]{fontenc}    
\usepackage{hyperref}       
\usepackage{url}            
\usepackage{booktabs,graphicx}       
\usepackage{amsfonts,bbm,lipsum}       
\usepackage{nicefrac}       
\usepackage{microtype}      
\usepackage{YK}
\usepackage{enumitem}
\setlist[itemize]{align=parleft,left=0pt..1em}
\setlist[enumerate]{align=parleft,left=0pt..1em}
\usepackage{caption}
\usepackage{subcaption,wrapfig}

\usepackage{xcolor}
\hypersetup{
    colorlinks,
    linkcolor={red!50!black},
    citecolor={blue!50!black},
    urlcolor={blue!80!black}
}

\newcommand{\bios}{\textsc{BiasBios}}
\newcommand{\cifar}{\textsc{CIFAR10}}
\def\doop{\mathrm{do}}

\def\NIE{\mathrm{NIE}}
\def\rNDE{\mathrm{rNDE}}
\def\TE{\mathrm{TE}}

\title{An Investigation of Representation and Allocation Harms in Contrastive Learning}

\author{Subha Maity \thanks{Corresponding author. Accompanying codes can be found in \href{https://github.com/smaityumich/CL-representation-harm}{https://github.com/smaityumich/CL-representation-harm}} \\
Department of Statistics\\
University of Michigan\\
Ann Arbor, MI \\
\texttt{smaity@umich.edu} \\
\And
Mayank Agarwal \\
IBM Research\\
MIT-IBM Watson AI lab\\
Cambridge, MA \\
\texttt{mayank.agarwal@ibm.com} \\
\AND
Mikhail Yurochkin \\
IBM Research\\
MIT-IBM Watson AI lab\\
Cambridge, MA \\
\texttt{mikhail.yurochkin@ibm.com} \\
\And
Yuekai Sun\\
Department of Statistics\\
University of Michigan\\
Ann Arbor, MI \\
\texttt{yuekai@umich.edu} \\
}

\begin{document}

\maketitle

\begin{abstract}
The effect of underrepresentation on the performance of minority groups is known to be a serious problem in supervised learning settings; however, it has been underexplored so far in the context of self-supervised learning (SSL). In this paper, we demonstrate that contrastive learning (CL), a popular variant of SSL, tends to collapse representations of minority groups with certain majority groups. We refer to this phenomenon as representation harm and demonstrate it on image and text datasets using the corresponding popular CL methods. Furthermore, our causal mediation analysis of allocation harm on a downstream classification task reveals that representation harm is partly responsible for it, thus emphasizing the importance of studying and mitigating representation harm. Finally, we provide a theoretical explanation for representation harm using a stochastic block model that leads to a representational neural collapse in a contrastive learning setting.
\end{abstract}

\section{Introduction}

\begin{wrapfigure}[14]{r}{0.4\linewidth}
\vspace{-0.22in}
    \centering
    \includegraphics[width=\linewidth]{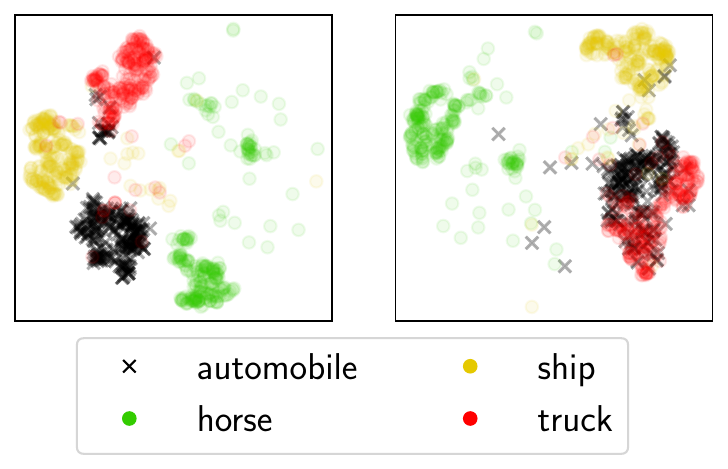}
    \vspace{-0.2in}
    \caption{tSNE visualization of CL representations for balanced (left) and \texttt{automobile} under-represented (right) \cifar\ data. }
    \label{fig:auto-underrepresented}
\end{wrapfigure}

The negative impacts of underrepresentation on the performance of minority groups have been extensively studied in supervised learning. Most prominent reports of algorithmic biases in automated decision-making pipelines \citep{angwin2016Machine,oneil2017Weapons,dastin2018Amazon} focus on disparities in resource allocation by supervised learning algorithms, and there are theoretical models that seek to elucidate the mechanisms behind the disparities \citep{fang2021exploring}. Such disparities cause \emph{allocative harms} to the samples from underrepresented groups, and there are many approaches based on resampling and reweighting \citep{he2008adasyn,ando2017deep,byrd2019effect}, enforcing invariance constraints \citep{agarwal2018Reductions}, re-calibrating the logits \citep{tian2020posterior} to alleviate such allocative harms in supervised learning algorithms.

In contrast to the voluminous literature on allocative harms, representation harms have received less attention. This is because representation harms are harder to measure and their effects are more diffuse and longer-term \citep{oneil2017Weapons}. 
Despite recent advances in representation learning 
\citep{bengio2013representation},
there is a lack of research on algorithmic biases in representation learning, especially studies that look into representation harms they may cause. 
In this paper, we seek to fill this gap in the literature by investigating the effects of group underrepresentation in contrastive learning (CL) algorithms \citep{chen2020simpleCL,chen2021exploring,he2020Momentum}. CL is a popular instance of self-supervised learning (SSL), an approach for learning representations that
can effectively leverage unlabeled samples to improve (downstream) model performances \citep{wang2021large,babu2021xls}. However, these datasets often exhibit inherent imbalances \citep{de-arteaga2019Bias} and it is important to know whether training a CL model on such datasets may detrimentally affect the quality of representations and cause downstream allocative and representation harms.

Here is a preview of the results of our controlled study of underrepresentation with the \cifar\ dataset. In Figure \ref{fig:auto-underrepresented} we plot 2D t-SNE embeddings \citep{van2008visualizing} of CL representations of vehicles in the dataset. We see that the underrepresentation of \texttt{automobile} images in CL training causes their representation to cluster with those of \texttt{trucks}. This is an instance of stereotyping \citep{abbasi2019Fairness} because samples from an underrepresented group are lumped together with those of a similar majority group; thus erasing the unique characteristics of the underrepresented. This is a form of representation harm \citep{crawford2017Trouble}, which can lead to downstream allocation harms (\eg\ missclassifications between \texttt{automobiles} and \texttt{trucks}). In this paper, we show that this mechanism of representation to downstream allocation harms is a general. This complements prior works, which show that underrepresentation leads to allocation harms without identifying the underlying mechanism. The rest of the paper is organized as follows.
\begin{enumerate}
\item In Section \ref{sec:representation-harm} we empirically show that underrepresentation leads to representation harms, as the CL representations of underrepresented groups collapse to semantically similar groups. 
\item In Section \ref{sec:theory} we develop a simple model of CL on graphs that exhibits a collapse of (the learned representations of) underrepresented groups to (those of) semantically similar groups. This suggests that the representation harms we measured are intrinsic to CL representations.
\item In Section \ref{sec:allocation-harm} we show via a (causal) mediation analysis that some of the allocation harms from a linear head/probe (trained on top of the CL representations) can be attributed to the representation harms in the CL representations. Thus, in order to mitigate the allocation harms from ML models built on top of CL representations, we must mitigate the harms in the learned representations.
\end{enumerate}

\subsection{Related works}

The paradigm of \textbf{self-supervised learning (SSL)} allows the use of large-scale unsupervised datasets to train useful representations and has drawn significant attention in modern machine learning (ML) research \citep{misra2020self,jing2020self}. It has found applications in a broad spectrum of areas, such as computer vision \citep{chen2020simpleCL,he2020Momentum}, natural language processing \citep{fang2020cert,liu2021tera,hsu2021hubert}, and many more. In this paper, we investigate the effect of underrepresentation in contrastive learning (CL) \citep{chen2020simpleCL,chen2021exploring,he2020Momentum}, which is a popular variant for self-supervised learning, that uses similar samples to learn SSL representations.
A more extensive discussion of previous work can be found in the surveys \citet{schiappa2023self,yu2023self,liu2022graph} for SSL and \citet{le2020contrastive,kumar2022contrastive} for CL. 
In this paper, we use SimCLR \citep{chen2020simpleCL} and SimSiam \citep{chen2021exploring} for \cifar\ dataset and SimCSE \citep{gao2021simcse} for \bios\ dataset to learn CL representations.   

Several works have \textbf{theoretically studied} the success of \textbf{self-supervised learning} \citep{arora2019theoretical,haochen2021provable,lee2020Predicting,tian2021understanding,tosh2021contrastive}. Our theoretical analysis of CL loss is partly motivated by \citet{fang2021exploring}, who showed that CL representations of a group collapse to a single vector. This phenomenon is known as \textbf{neural collapse}, and it also occurs in supervised learning \citep{papyan2020prevalence,NEURIPS2021_f92586a2,fang2021LayerPeeled,fang2021exploring,hui2022limitations,han2021neural} and transfer learning \citep{galanti2021role}. Our analysis differs from \cite{fang2021exploring} because they study neural collapse when the dataset is balanced, while we focus on imbalanced datasets. This is crucial for studying allocation and representation harms caused by underrepresentation. We show that when the dataset is imbalanced, neural collapse still occurs, but the collapsed points no longer possess a symmetric simplex structure. Instead, the relative positions of the collapsed points depend on the degree of underrepresentation in an intricate manner; we explore the resulting structure in Section \ref{sec:node2vec-thm}. 

\citet{liu2021selfsupervised} showed that under group imbalance, predictive models based on SSL representations cause less allocation harm than models trained in a supervised end-to-end manner. However, their work does not preclude allocation harms by SSL; merely allocation harms are reduced compared to supervised learning. We build on their work by studying the allocation and representation harms caused by SSL.

A discussion of previous work on mitigating representation harm is deferred to the second last paragraph of Section \ref{sec:discussion}.

\section{Under-representation causes stereotyping}
\label{sec:representation-harm}

In Figure \ref{fig:auto-underrepresented} we observe that underrepresentation of \texttt{automobiles} causes its CL representations to collapse with \texttt{trucks}. To delve deeper into this phenomenon, in this section, we consider two cases of underrepresentation: (i) a controlled study with \cifar\ dataset, where we emulate the effect of underrepresentation by systemically subsampling a class, and (ii) on \bios\ dataset, which is naturally imbalanced. 

\subsection{Controlled study of class underrepresentation}
\label{sec:cifar10}

\begin{wrapfigure}[20]{r}{0.5\linewidth}
\vspace{-0.1in}
    \centering
    \includegraphics[width=\linewidth]{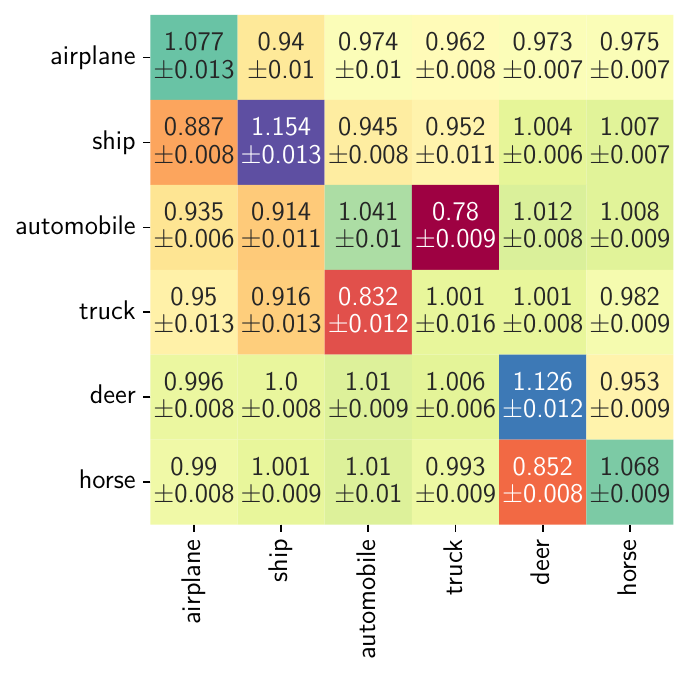}
    \vspace{-0.3in}
    \caption{Representation harms in \cifar\ over 10 repetitions.}
    \label{fig:simclr-RH-100-fold}
    \vspace{-0.3in}
\end{wrapfigure}

\subsubsection{Experimental setup}

\paragraph{Dataset:}
For our controlled study, we consider \cifar\ \citep{krizhevsky2009learning}, a well-known visual benchmark dataset that has 10 classes (\texttt{airplane}, \texttt{cars}, etc.) and both the training and the test datasets are balanced (equal number of images in all classes). To simulate underrepresentation, we randomly subsample 1\% of the images for one of the classes when training our CL models. We denote this dataset as $\cD_k$, where $k$ is the class that is undersampled.  Furthermore, we train a CL model on the balanced dataset ($\cD_{\text{bal}}$) as a baseline.  

\paragraph{Models and representations:} We use the ResNet-34 backbone and train on both balanced and underrepresented training datasets using two well-known CL approaches: SimCLR \citep{chen2020simpleCL} and SimSiam \citep{chen2021exploring}. We will refer to the model trained with a balanced (resp. underrepresented) dataset as a balanced (resp. underrepresented) model. 
In the main text, we present results for SimCLR. The results for SimSiam are deferred to Appendix \ref{sec:simsiam-cifar-supp} and are quite similar to the SimCLR results.  
Further details and results of the SimCLR implementation can be found in Appendix \ref{sec:simclr-cifar-supp}.

\subsubsection{Representation harm}

\label{sec:rep-harm}
In Figure \ref{fig:auto-underrepresented} we see that the underrepresentation in \texttt{automobile} class causes its CL representations to collapse with \texttt{trucks}, which is a semantically similar class. Here, we further investigate this behavior, \ie, whether this phenomenon persists among other semantically similar pairs. 

\paragraph{Metric:}
We measure the representation harm between a pair ($l$, $m$) when the $k$-th class is underrepresented as
\begin{equation} \label{eq:rep-harm-metric}
\begin{aligned}\textstyle
     \text{RH}(l, m; k) = \frac{ {\bD(l, m; f_k)}}{\bD(l, m; f_\text{bal})},  ~
      \bD(l, m; f_k) = \frac{1}{n_l n_m} \sum_{\substack{i \in[n_l],\\j \in [n_m]}}  \big\{1 - \cos\big(f_k(x_{l, i}), f_k(x_{m, j})\big)\big\}\,, 
\end{aligned}
\end{equation} where $f_{\text{bal}}$ (resp. $f_{k}$) is the balanced (resp. underrepresented) CL model  trained on $\cD_\text{bal}$ (resp. $\cD_k$). $\text{RH}(l, m; l) < 1$ implies that underrepresentation in class $l$ causes its CL representation to collapse with class $m$. 
In Figure \ref{fig:simclr-RH-100-fold} we plot the representation harm metrics, where the $(l, m)$-th entry is $\text{RH}(l, m; l)$.  We present RH metrics for the four vehical classes (\texttt{airplane}, \texttt{ship}, \texttt{automoble} and \texttt{truck}), and two animal classes (\texttt{deer} and \texttt{horse}).
The other RH metrics can be found in Appendix \ref{sec:simclr-RH-AH-supp}.

\paragraph{Results:}  In Figure \ref{fig:simclr-RH-100-fold} we observe the worst case of representation harm between \texttt{automobile} and \texttt{truck} classes. The RH metric for this pair is the lowest ($0.78 \pm 0.009$) when \texttt{automobile} is undersampled, and second lowest ($0.832 \pm 0.012$) when \texttt{truck} is undersampled. This means underrepresentation in either \texttt{automobile} or \texttt{truck} results in a collapse in their CL representations with each other. Similar behavior is also observed between pairs (\texttt{airplane}, \texttt{ships}) (RH metrics are $0.887 \pm 0.008$ and $0.94  \pm 0.01$) and (\texttt{deers}, \texttt{horses}) (RH metrics are $0.852 \pm 0.008$ and $0.953 \pm 0.009$). Since (\texttt{automobiles}, \texttt{trucks}) are vehicles on road , (\texttt{ships}, \texttt{airplanes}) are vehicles with blue (water/sky) backgrounds, and (\texttt{deers}, \texttt{horses}) are mammals with green backgrounds, these pairs are semantically similar pairs, which affirms our intuition that CL representations of an underrepresented class collapse with the representations of a similar class. We shall see later in Section \ref{sec:allocation-harm-cifar10} that these representation harms cause allocation harm, \ie\ reduction in downstream classification accuracies.

\subsection{Effects of under-representation in the wild} \label{sec:bios}
An important premise of self-supervised learning is the ability to train on large amounts of data from the Internet with little or no data curation. Such data will inevitably contain many underrepresented groups. In this experiment, we study the potential harms of CL applied to data obtained from Common Crawl, a popular source of text data for self-supervised learning.

\subsubsection{Experimental setup}

\begin{wrapfigure}[21]{r}{0.6\linewidth}
\vspace{-0.3in}
    \centering \includegraphics[width=\linewidth]{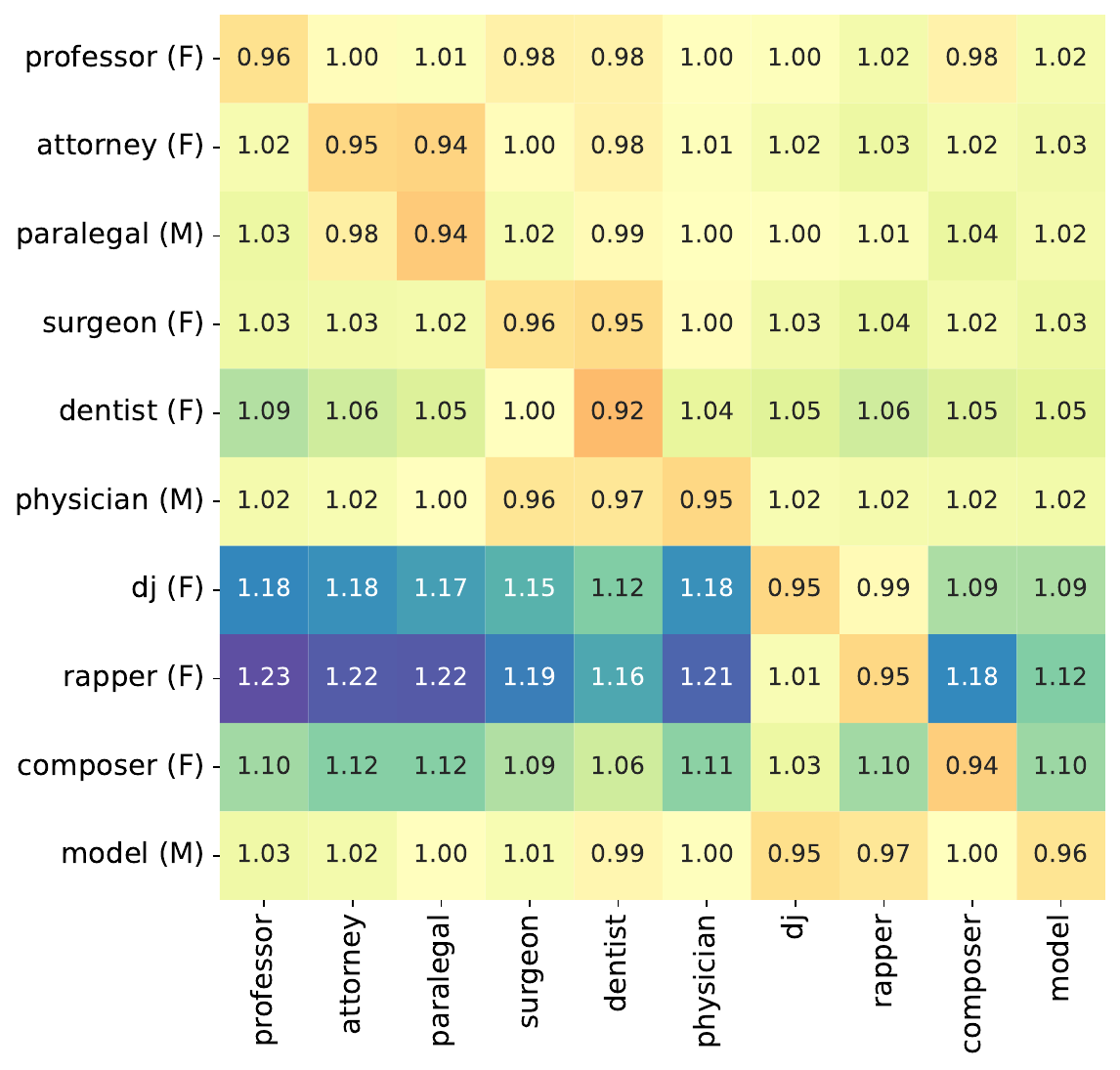}
    \vspace{-0.2in}
    \caption{Gender RH in \bios\ dataset.}
    \label{fig:representation-harm-bios}
\end{wrapfigure}

\paragraph{Dataset:} We consider \bios\ dataset \citep{de-arteaga2019Bias} which consists of around 400k online biographies in English extracted from the Common Crawl data. These biographies represent 28 occupations appearing at different frequencies (\texttt{professor} being the most common, and \texttt{rapper} the least common). In addition, many occupations are dominated by biographies of either male or female gender (identified using the pronouns in the biographies), mimicking societal gender stereotypes. Please see Appendix \ref{sec:bios_supp} for additional details.

\paragraph{Model and representations:} We obtain CL representations by fine-tuning BERT \citep{devlin2018BERT} with SimCSE loss \citep{gao2021simcse}, an analog of SimCLR for text data, that uses dropout in the representation space instead of image augmentations. We use a random 75\% of samples for training the SimCSE representations and the remaining data for computing all metrics reported in the experiments. SimCSE representations are trained using the source code from the authors in the unsupervised mode and with the MLP projection layer \citep{gao2021simcse}.

\subsubsection{Representation harm}

As discussed previously, we expect the representations of underrepresented groups to collapse towards the groups that are most similar to them. For example, consider the class \texttt{surgeon} which consists of 85\% male and 15\% female biographies. Will the biographies of underrepresented female \texttt{surgeons} be assigned representations similar to male \texttt{surgeons} or to female biographies of a different but related occupation such as \texttt{dentist}? Although none of these outcomes are desirable, the latter can cause representation harms associated with gender stereotyping.

\paragraph{Metric:} To measure the representation harm and answer the aforementioned question we consider a variant of the metric based on the average cosine distance used in the \cifar\ experiment. Let $\bD(l_g, m_{g'})$ be the average cosine distance between the representations of gender $g$ biographies from occupation $l$ and gender $g'$ biographies from occupation $m$ (analogous to \eqref{eq:rep-harm-metric} where we omit the dependency on the model $f$ since we have a single CL model in this experiment). Let $g$ be the underrepresented gender for occupation $l$, we define gender representation harm (GRH) for occupations $l$ and $m$ as:
\begin{equation}\textstyle
\label{eq:rep-harm-bios}
\text{GRH}(l, m) = \frac{\bD(l_g, m_g)}{\bD(l_g, l_{g'})}.
\end{equation}
$\text{GRH}(l, m) < 1$ implies that the learned representations for the underrepresented gender in occupation $l$ are \emph{closer} to representations of occupations $m$ for the same gender than they are to the representations for the different gender within the same occupation. $\text{GRH}(l, m) < 1$ for $l \neq m$ can be interpreted as a warning sign of gender stereotyping in representations.

\paragraph{Results:} We present the GRH results for a subset of occupations in Figure \ref{fig:representation-harm-bios} (F and M in the row names indicate the underrepresented gender for the corresponding occupation; complete results are in Appendix \ref{subsec:grh_bios_all}). We focus on the off-diagonal entries which should be >1 when there is no representation harm. We observe several deviations from this rule, especially for occupation pairs (\texttt{attorney}, \texttt{paralegal}) and (\texttt{surgeon}, \texttt{dentist}). Specifically, GRH$(\texttt{attorney}, \texttt{paralegal})$ of 0.94 suggests that representations of female attorneys are closer to representations of female paralegals than they are to that of male attorneys. The two occupations are highly related and such representation harm can result in disadvantaging female attorneys. Analogously, GRH$(\texttt{surgeon}, \texttt{dentist})$ of 0.95 corresponds to a similar problem for a pair of occupations in the medical domain.

\section{Asymptotic analysis of CL representations}
\label{sec:theory}

In this section, we show theoretically that the representation harms in contrastively learned representations are generic and are not specific to certain datasets. Here, we focus on a generic contrastive learning loss \citep{chen2020simpleCL,wang2020understanding}:
\begin{equation}
    \begin{aligned}
        \min\nolimits_{\Phi:\cX\to\bS^{d-1}}\bL_\text{CL}(\Phi(X)), \\\textstyle
    \bL_{\text{CL}}(V) \triangleq - \frac1n \sum_{i\in [n]}  \frac{1}{\sum_{j\in[n]}e(i,j)} \sum_{j\in [n]} e(i, j) \log \Big\{\frac{\exp(\nicefrac{v_i^\top v_j}{\tau})}{\frac1n \sum_{l \in [n]}\exp(\nicefrac{v_i^\top v_l}{\tau})}\Big\},\label{eq:node2vec}
    \end{aligned}
\end{equation}
where $\bS^{d-1}$ is the unit sphere in $\reals^d$, the rows of $\Phi(X)\in\reals^{n\times d}$ are the learned representations of the samples $X_i$, and $v_i \in \reals^d$ is the (learned) representation of the $i$-th sample. 
Note that the loss depends on the raw inputs $X_i$'s only through their similarities $e(i,j)$. We shall exploit this fact to simplify the subsequent theoretical developments.
In practice, $\Phi$ is usually parameterized as a neural network. 

To isolate the effects of the loss (from other inductive biases encoded in the architecture of $\Phi$ and the training algorithm), we make the layer-peeled assumption \citep{fang2021LayerPeeled,fang2021exploring} that $\Phi$ can produce any point in $\bS^{d-1}$. This assumption simplifies the problem to optimizing directly on the outputs of $\Phi$ (instead of the parameters of $\Phi$).
\begin{equation}
\min\nolimits_{v_i\in\bS^{d-1}}\bL_\text{CL}(V),
\label{eq:node2vec-layer-peeled}
\end{equation}
where the optimizers $v_i^\star$ corresponds to the learned representation of $X_i$. Since the CL loss depends only on the similarities between samples, we directly impose a probabilistic model on the similarities (instead of imposing assumptions on the distribution of samples). This simplifies the theoretical developments. 

\paragraph{Stochastic block model:}
We formalize the similarities between samples as a similarity graph on the samples and impose a stochastic block model (SBM) \citep{holland1983Stochastic} on the similarity graph.  Typically, an SBM with $K$ blocks is described by two parameters: (1) $\Pi = [\pi_1, \dots, \pi_K]^\top$ the probabilities of a data point or node belonging to each of the blocks, and (2) $A = [[\alpha_{k, k'}]]_{k, k'\in[K]}$, a matrix that describes connectivity between blocks, where $\alpha_{k, k'}$ is the probability that a node from block $k$ is connected to a node from block $k'$. The $\pi_k$' s controls the underrepresentation of a block and $\alpha_{k, k'}$'s determines the similarity between blocks.  A sample of size $n$ is drawn from $\text{SBM}(\Pi, A)$ in the following manner: 
\begin{enumerate}
    \item We generate $n$ nodes from a multinomial distribution with probability vector $\Pi$. Let us denote $\{Y_i\}_{i = 1}^n \subset [K]$ as the block annotations, \ie\ the $i$-th node belongs to the $Y_i$-th block. Note that $Y_i$ are i.i.d. $\text{categorical}(\Pi)$. 
    \item For each pair of nodes, we generate $e(i, i') \mid Y_i, Y_{i'} \sim \text{Bernoulli}(\alpha_{Y_i, Y_{i'}})$, where $e(i, i')$ indicates whether the nodes $i$ and $i'$ are connected by an undirected edge. In the context of CL from images, such edges imply that neighboring nodes $i$ and ${i'}$ are augmentations of the same or similar images.
\end{enumerate}
The SBM has two desirable properties:
\begin{enumerate}
    \item An inherent \emph{group structure} that is often associated with downstream classification tasks. For example, the classes in \cifar\ dataset can be considered as groups. 
    \item A notion of \emph{connectivity} among groups that explains which of them are similar. As we shall see later, this is a key factor in identifying the majority groups to which representations of an underrepresented group can collapse. Recall, in the example of \cifar\ (Section \ref{sec:allocation-harm}) we see that the \texttt{automobile} and \texttt{truck} classes are connected to each other, and underrepresentation of either of them results in the collapse of their CL representations. 
\end{enumerate}

\subsection{Simulations with the stochastic block model}
\label{sec:SMB-simulation}

To understand how underrepresentation may cause representation harm in CL representations, we conduct a simple experiment on SBM data. We consider an SBM with three blocks. The connectivity probabilities are provided in the left plot of Figure \ref{fig:SBM-mean-cos}, where only the first two groups are connected with each other. Additional details are provided in Appendix \ref{sec:supp_sbm}. 

\begin{figure}
    \centering
    \includegraphics[width=\linewidth]{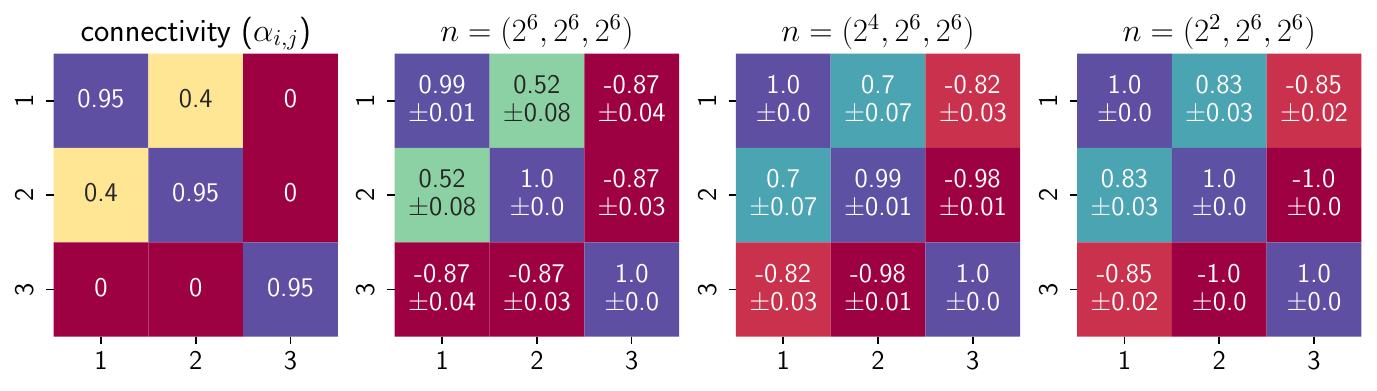}
    \vspace{-0.1in}
    \caption{\emph{Left:} connectivity across groups, \emph{middle-left, middle-right and right:} average cosines and their standard deviations across groups.}
    \label{fig:SBM-mean-cos}
\end{figure}

In our experiment, we underrepresent the first block by factors of $2^{-2}$ and $2^{-4}$ (the sample size for the first group is reduced from $2^6$ to $2^4$ and $2^2$) and investigate its representation harm. For this purpose, we obtained the CL representations from \eqref{eq:node2vec-layer-peeled}. Note that the generic CL cost function is exactly the popular node2vec cost function \citep{grover2016node2vec} for graph representation learning, so our results also have implications in graph representation learning. In Figure \ref{fig:SBM-mean-cos} we present the means and standard deviations for the cosines of the CL representations between groups.
Some notable observations are described below.
\begin{itemize}
    \item The diagonals in Figure \ref{fig:SBM-mean-cos} demonstrate that all $v_i^\star$ vectors within a block collapse into a single vector, as they have cosines very close to one. This phenomenon is known as \emph{neural collapse} \citep{papyan2020prevalence} and has been observed by \citet{fang2021exploring} in analyzing CL representation with balanced data. Further related discussions are deferred to Section \ref{sec:node2vec-thm}.
    \item The same plots suggest that the CL representations of the first two groups get closer when the first group is underrepresented, as seen in their increased cosine similarity (it increases from $0.52 (\pm  0.08)$ to $0.7 (\pm 0.07)$ in the middle right and $0.83 (\pm 0.03)$ in the right plot).  This is a form of representation harm, as the representations of the first two groups, which are similar to each other, partially collapse in terms of their CL representations (RH metric in \eqref{eq:rep-harm-metric} are $\text{RH}(1, 2; 1) = 0.48$  and $0.35$). Additionally, the second and third groups become further apart as their cosines decrease from $-0.87 (\pm 0.03)$ to $-0.99 (\pm 0.0)$ and $-1 (\pm 0.0)$.
\end{itemize}

\subsection{Asymptotic analysis of CL representations for stochastic block models} \label{sec:node2vec-thm}
Here, we verify the two phenomena that we observe in our simulation with the SBM dataset: (1) the collapse of CL representations within a block and (2) the representation harm between blocks with high connectivity. For this purpose, we perform an asymptotic analysis on the optimal CL representations $v_i^\star$ at $n$ going to infinity. We require the following assumption in our analysis. 
\begin{assumption} \label{assmp:rep-convergence}
Denoting the representation variables of the $k$-th block as $\{v_{k, j}\}_{j = 1}^{n_k}$, where $n_k = \#\{Y_i = k\}$ and for each $k$ we assume that $\nicefrac{1}{n_k} \sum_{j = 1}^{n_k} v_{k, j}$ converges as $n_k \to \infty$.  
\end{assumption}

The above assumption simply requires that the representation variables in the optimization of \eqref{eq:node2vec-layer-peeled} that belong to the same block do not fluctuate too much, which is a rather minimal assumption. With this assumption, we state our main result. 
\begin{theorem} \label{th:layer-peeled-node2vec}
Assume that $\sum_{k' \in [K]} \pi_{k'} \alpha_{k, k'} > 0$ for each $k$. 
    Then at $n\to \infty$ the optimum CL representations obtained from the minimization of problem in \eqref{eq:node2vec-layer-peeled} satisfy the following:  $v_i^\star \stackrel{a.s.}{=} h_{Y_i}^\star$,
    where $\stackrel{a.s.}{=}$ denotes almost sure equality and $\{h_k^\star\}_{k \in [K]}$ is a minimizer for \begin{equation} \label{eq:node2vec-neural-collapse}
    \textstyle  \underset{h_k\in \bS^{d-1}}{\min}  - 
    \sum_{k_1 = 1}^K \pi_{k_1}
    \frac{ 
    \sum_{k_2 = 1}^K \pi_{k_2}  \alpha_{k_1, k_2} (\nicefrac{ h_{k_1}^\top h_{k_2}}{\tau})
      }{
      \sum_{k_2 = 1}^K \pi_{k_2} \alpha_{k_1, k_2}
      } 
      +
      \sum_{k_1 = 1}^K \pi_{k_1}
     \log \Big \{
    \sum_{k_3 = 1}^K  \pi_{k_3}
    e^{
    \nicefrac{h_{k_1}^\top h_{k_3}}{\tau}
    }
    \Big\}\,.
    \end{equation}
Note that the objective in \eqref{eq:node2vec-neural-collapse} is a weighted version of the generic CL objective in \eqref{eq:node2vec-layer-peeled} applied to common group-wise representations $h_{Y_i}$. 
\end{theorem}

This theorem shows that all samples from the same block have the same representation as $n\to\infty$; \ie\ all $v_i^\star$'s converge to their corresponding $h_{Y_i} ^\star$s. This phenomenon is known as \emph{neural collapse}. It has been studied in the context of a layer peeled model for supervised learning \citep{papyan2020prevalence,NEURIPS2021_f92586a2,mixon2020neural,lu2020neural}, and in CL representations with balanced data \citep{fang2021exploring}. \citet{fang2021exploring}'s results are most similar to ours, but do not elicit the effects of dataset imbalance on the geometry of the collapsed representations. 

\begin{wrapfigure}[11]{r}{0.68\textwidth}
    \centering
    \vspace{-0.2in}
    \includegraphics[width=\linewidth]{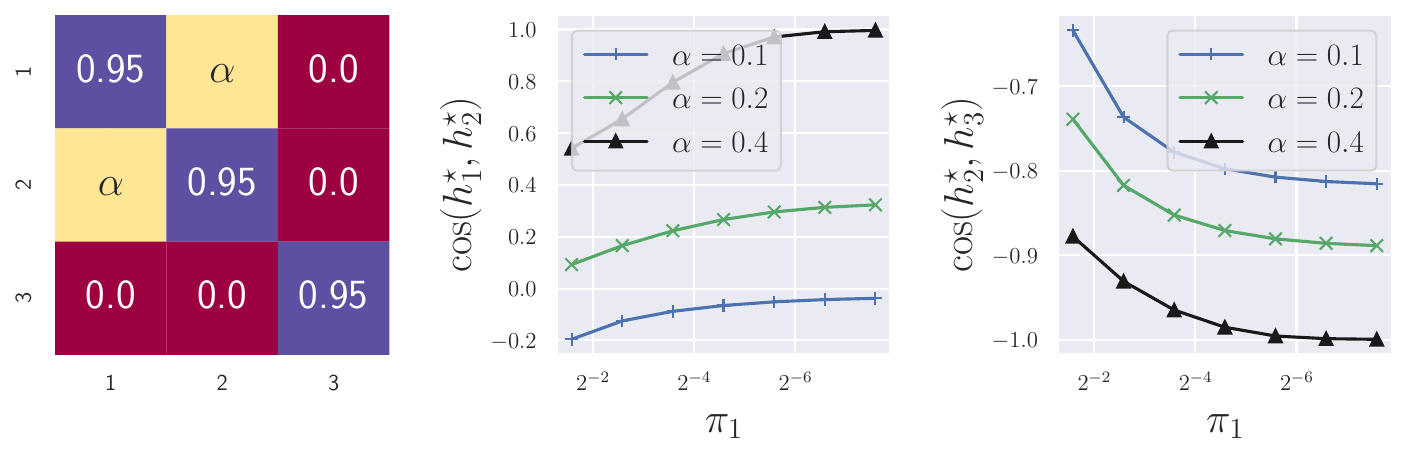}
    \vspace{-0.25in}
    \caption{\emph{Left:} connectivity, and \emph{middle and right:}  cosine similarities between groups (1, 2) and (2, 3).}
    \label{fig:rep-harm-SBM2}
\end{wrapfigure}

\paragraph{Representation harm:}  To understand the representation harm we reconsider the setup of the synthetic experiment in Section \ref{sec:SMB-simulation}. The SBM has three blocks, where for a given $\pi_1$ we set $\pi_2= \pi_3 = (1 - \pi_1)/ 2$. In the leftmost plot of Figure \ref{fig:rep-harm-SBM2} we present the connectivity matrix, and as before, only the first two blocks are connected with probability $\alpha$.  In this setup, we obtain representations $h_k^\star$ from \eqref{eq:node2vec-neural-collapse} and plot their cosine similarities for pairs (1, 2) and (2, 3) in Figure \ref{fig:rep-harm-SBM2}. The harm in representation becomes more prominent between the first two groups with the severity of the underrepresentation of the minority group, as their cosine increases in the middle plot of Figure \ref{fig:rep-harm-SBM2}. In fact, their representations become identical for $\pi_1 < 2^{-6}$ when there is sufficient connectivity between the first two groups ($\alpha = 0.4$). Additionally, the two disconnected majority groups (second and third groups) get further apart, and at the extreme ($\pi_1 > 2^{-6}$) they become exactly opposite (cosine is $-1$).

Finally, when the first two groups are less connected ($\alpha = 0.1$), their representations do not collapse, even for a severe underrepresentation of the first group (cosine is less than zero when $\pi_1 < 2^{-6}$). This relates to an observation in \cifar\ case-study in Figure \ref{fig:RH-100-fold-full} in Appendix \ref{sec:simclr-RH-AH-supp}, that the \texttt{dog} class, which is arguably disconnected from other classes, suffers the least allocation harm when it is underrepresented.

\paragraph{Practical implications:} Reiterating our previous discussions, our analysis suggests that harm in representation occurs between two semantically similar groups when either of them is underrepresented. Therefore, to mitigate the issue, practitioners should consider a combination of the following strategies; 1) reweighing the loss to counteract underrepresentation, an approach considered in \citet{liu2021selfsupervised,zhou2022Contrastive}, and 2) a "surgery" on connectivity, similar to the approach in \citet{ma2021conditional}. We defer our further discussion of the mitigation of harm to Section \ref{sec:discussion}.

\section{Representation harms in CL representations cause allocation harms}
\label{sec:allocation-harm}

\begin{wrapfigure}[7]{r}{0.25\textwidth}
    \centering
    \vspace{-0.2in}
    \includegraphics[width=0.25\textwidth]{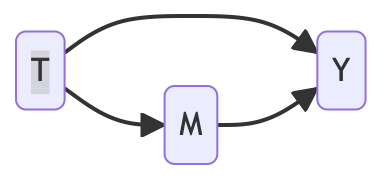}
    \caption{Basic model for causal mediation analysis}
    \label{fig:causal-mediation-analysis-model}
\end{wrapfigure}

In Section \ref{sec:representation-harm} we observed that underrepresentation of a group causes its CL representation to collapse with symantically similar groups. 
Through a causal mediation analysis (CMA) \citep{pearl2022direct}, in this section, we show that this is partly responsible for allocation harm (AH) in downstream classification tasks. Thus, to fully mitigate downstream allocation harm, practitioners must address this representation harm.

\subsection{Background on causal mediation analysis}

The goal of CMA is to decompose the causal effect of a treatment on an outcome into a direct (causal) effect that goes directly from the treatment to the outcome and an indirect effect that goes through other variables in the causal graph. In this section, we consider the basic CMA graph shown in Figure \ref{fig:causal-mediation-analysis-model}. Here, treatment $T$ is whether the dataset is undersampled, $M$ is the CL representations, and $Y$ is a measure of downstream allocation harm (\eg\ misclassification rate). 

The two effects that we evaluate are the \textbf{natural indirect effect (NIE)} and the \textbf{reverse natural direct effect (rNDE)}. It is easily seen that they add up to the total effect $\TE\triangleq \Ex\big[Y_{\doop(T=1)}\big] - \Ex\big[Y_{\doop(T=0)}\big]$:
\begin{equation}
\TE = \underbrace{\Ex\big[Y_{\doop(T=1,M=1)}\big] - \Ex\big[Y_{\doop(T=0,M=1)}\big]}_{-\rNDE} + \underbrace{\Ex\big[Y_{\doop(T=0,M=1)}\big] - \Ex\big[Y_{\doop(T=0,M=0)}\big]}_{\NIE}
\label{eq:TE-decomposition}
\end{equation}

Note that \eqref{eq:TE-decomposition} is not the standard decomposition of TE in CMA; the standard decomposition expresses TE as the sum of the NDE and rNIE. Here, the NIE is the downstream allocation that can be attributed to (representation harms) in the CL representation, while the rNDE is the downstream allocation harms directly caused by underrepresentation.

\subsection{Mediation analysis on controlled study}
\label{sec:allocation-harm-cifar10}
In our controlled study with \cifar\ dataset in Section \ref{sec:cifar10} we observed that semantically similar pairs such as (\texttt{automobiles}, \texttt{trucks}), (\texttt{airplanes}, \texttt{ships}), and (\texttt{deers}, \texttt{horses}) suffer from representation harm when one of the classes is underrepresented.
Following this, using mediation analysis, we show that it may cause allocation harm in a downstream classification, thus emphasizing its importance in mitigating allocation harm.

\paragraph{Setup:} Suppose that treatment $T = 1$ is the underrepresentation of class $k$. We train a linear head with a randomly chosen 75\% of the test data, considering two scenarios: (1) it is trained on a balanced dataset ($T = 0$), and (2) it is trained on an imbalanced dataset where the class $k$ is subsampled to $1\%$ of its original size ($T = 1$).  Recalling that $f_{\text{bal}}$ (resp. $f_{k}$) denotes the CL model trained on balanced (resp. class $k$ underrepresented) training dataset, we denote training a linear head on $f_{\text{bal}}$ as $M = 0$ and the same with $f_k$ as $M = 1$. A CL model coupled with a linear head builds the final image classifier, which we always evaluate on the remaining 25\% of the test dataset. 

\paragraph{Metrics:} We consider three classifiers, where a linear head is trained on top of: (1) $f_{\text{bal}}$ using balanced data ($T = 0, M = 0$), which we denote as $\hat y_{0, 0}$, (2) $f_k$ using balanced data ($T = 0, M = 1$), denoted as $\hat y_{0, 1}$, and finally (3) $f_k$ using imbalanced data $(T = 1, M = 1)$, denoted as $\hat y_{1, 1}$. With all the notations in hand, the total effect (resp. natural indirect effect) due to the underrepresentation of the class $k$ in classifying a sample $x$ with the true class $y = k$ as $\hat y = l$ is
\begin{equation} \label{eq:TE-NIE}
\begin{aligned}
     \text{TE}(k, l) &= P(\hat y^{(k)}_{1, 1}(x) = l \mid y = k) - P(\hat y_{0, 0}(x) = l \mid y = k)\,\\
     \text{NIE}(k, l) &= P(\hat y^{(k)}_{0, 1}(x) = l \mid y = k) - P(\hat y_{0, 0}(x) = l \mid y = k)\,.
\end{aligned}
\end{equation}

Finally, the corresponding reverse natural direct effect is calculated as $-\text{rNDE}(k, l) = \text{TE}(k, l) - \text{NIE} (k, l)$. For $k \neq l$ the
$\text{TE}(k, l) > 0$ indicates how much more often the class $k$ is mistaken for the class $l$ when $k$ is underrepresented and $\text{NIE}(k, l)$ is the part that is caused by representation harm. We present the $\text{TE}$, $-\text{rNDE}$ and $\text{NIE}$ as heatmap plots in Figure \ref{fig:AH-100-fold}, where $\text{TE}(k, l)$ is the entry of the $(k, l)$-th cell and similarly for $-\text{rNDE}$ and $\text{NIE}$. As in Section \ref{sec:rep-harm}, we present the metrics for only six classes (four vehicles and two animals). The remaining ones are provided in Appendix \ref{sec:simclr-RH-AH-supp}. 

\begin{figure}
    \centering
    \includegraphics[width=\linewidth]{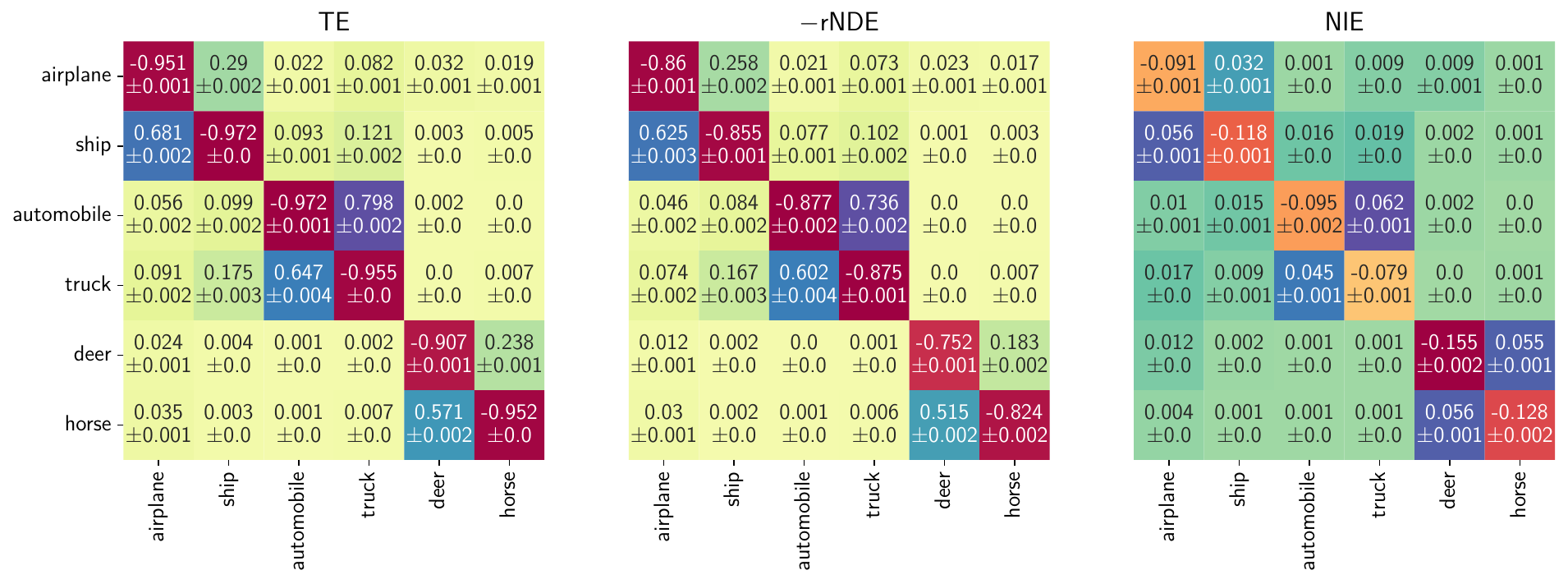}
    \caption{TE (\emph{left}), $-$rNDE (\emph{middle}) and NIE (\emph{right}) for \cifar\ dataset over 10 repetitions.   }
    \label{fig:AH-100-fold}
\end{figure}

\paragraph{Results:}  In Figure \ref{fig:AH-100-fold} we mainly focus on the NIE, as this part of allocation harm caused by harm in CL representation.  When underrepresented, \texttt{automobiles} are mistaken as \texttt{truck} most often (off-diagonal TE is highest at $0.798 \pm 0.002$). Additionally, the allocation harm caused by harm in CL representations is the highest in this case (the highest value observed for the NIE metric is at $0.062 \pm 0.001$), which is related to the highest representation harm observed between them (RH metric is $0.78 \pm 0.009$). Furthermore, the harm in CL representations causes a significant allocation harm for semantically similar pairs (\texttt{automobiles}, \texttt{trucks}), (\texttt{airplanes}, \texttt{ships}), and (\texttt{deers}, \texttt{horses}) as their NIE metrics are significantly higher. This aligns with our observations in Section \ref{sec:cifar10}, as these pairs suffer significant representation harm due to their underrepresentation.

Additionally, representation harm causes a significant reduction in downstream accuracy for these classes. This is observed in the diagonal NIE metrics in Figure \ref{fig:AH-100-fold}, which is the highest for \texttt{deer} ($15.5 \pm 0.2 \%$) and is at least $7.9 \pm 0.1\%$.  This emphasizes that one cannot completely mitigate allocation harm without addressing the harm in CL representations.

For the \bios\ case study performing mediation analysis is challenging due to the existing natural imbalance, which makes it infeasible to create a balanced dataset and thus train a balanced CL model. It is, however, possible to quantify allocation harm, which we investigate in Appendix \ref{sec:AH-bios-supp}.

\section{Summary and discussion}
\label{sec:discussion}

We studied the effects of underrepresentations on contrastive learning algorithms empirically on the \cifar\ and \bios\ datasets and theoretically in a stochastic block model. We find (both theoretically and empirically) that the CL representations of an underrepresented group collapse \emph{ to a semantically similar group}. Although prior work shows that classifiers trained on top of CL representations is more robust to underrepresentation than supervised learning \citep{liu2021selfsupervised} in terms of downstream allocation harms, we show that the CL representations themselves suffer from representation harms: the CL representations of an underrepresented group collapse to a semantically similar group. To reconcile our results with prior work, we decompose the downstream allocation harms in classifiers trained on top of CL representations into a direct effect and an indirect effect mediated by the CL representations via a causal mediation analysis. Our results show that it is necessary to address the representation harms in CL representations in order to eliminate allocation harms in classifiers built on top of CL representations.

\paragraph{Attempts to mitigate representation harms in CL:} Broadly speaking, the issue of underrepresentation is in conflict with one of the key advantages of CL or SSL: their ability to use large uncurated datasets from the Internet. Our results suggest broad adoption of SSL can lead to representations harms to underpriviledged groups. This corroborates recent empirical results on algorithmic biases \citep{wired2022dalle,naik2023social} in prompt-guided generative models powered by CLIP \citep{radford2021learning} (e.g., DALL$\cdot$E 2), which uses a variant of CL loss. To address these issues, we need SSL methods that can account for underrepresentation \emph{without requiring group annotations}. Prior works have proposed several such methods
\citep{liu2021selfsupervised,zhou2022Contrastive,assran2022hidden}, but their evaluation metrics are limited to downstream accuracy and thus their effectiveness in mitigating representation harms remains unexplored.
We investigate the representations learned with one of these methods, 
boosted contrastive learning (BCL) \citep{zhou2022Contrastive}. In a case study on \cifar\ dataset in Figure \ref{fig:bcl-part} (further details are in Appendix \ref{sec:BCL}), we find that BCL may not be sufficiently effective to completely mitigate the representation harm.
Although the RH metrics between the pairs (\texttt{automobiles}, \texttt{trucks}) increase from $0.78 \pm 0.009$ and $0.832 \pm 0.012$ (in Figure \ref{fig:simclr-RH-100-fold}) to $0.829 \pm 0.001$ and $0.897 \pm 0.001$ (in Figure \ref{fig:bcl-RH-part}),
BCL does not completely mitigate the harm in representations. As a result, accuracies for these classes still suffer non-trivial reductions, as observed in the diagonal NIE metrics in Figure \ref{fig:bcl-AH-part} (the reduction in accuracy is $5.3\pm 0.2 \%$ (resp. $3.7 \pm 0.01\%$) for \texttt{automobiles} (resp. \texttt{trucks})). These results show that the elimination of representation harms in CL remains a pressing open problem.

\begin{wrapfigure}[17]{r}{0.68\linewidth}
\vspace{-0.1in}
    \begin{subfigure}[b]{0.33\textwidth}
         \centering
         \includegraphics[width=\textwidth]{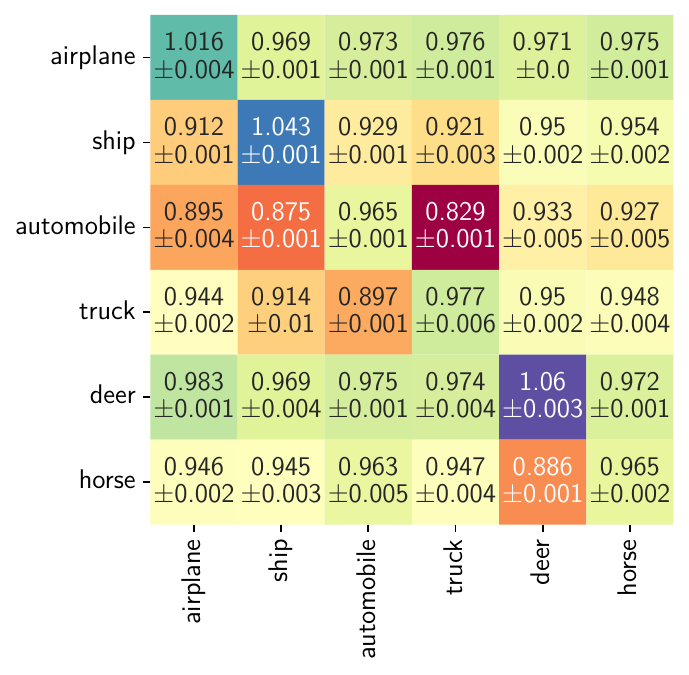}
         \caption{Representation harm}
         \label{fig:bcl-RH-part}
     \end{subfigure}
     \hfill
     \begin{subfigure}[b]{0.33\textwidth}
         \centering
         \includegraphics[width=\textwidth]{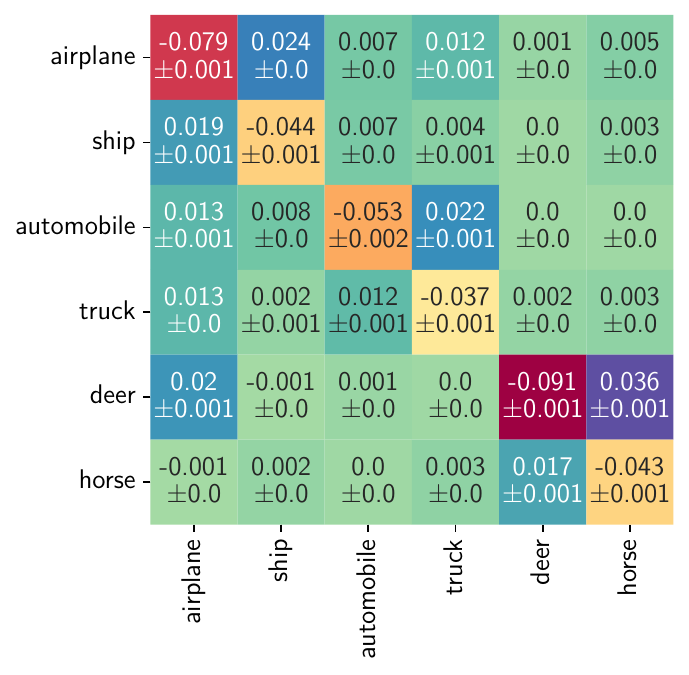}
         \caption{NIE}
         \label{fig:bcl-AH-part}
     \end{subfigure}
        \caption{Representation harm and NIE in allocation harm for BCL representations in \cifar.}
        \label{fig:bcl-part}
\end{wrapfigure}

Our theoretical analysis suggests that practitioners should consider a combination of the following strategies to mitigate harm in representation; (1) reweighing the loss to counteract underrepresentation and (2) a ``surgery'' on connectivity, both of which require group annotations. Since they are not available in most CL applications, it would be interesting to attempt to combine the two using a proxy for group annotations. There have been many efforts to improve performance in minority groups without group annotations in the supervised learning setting \citep{hashimoto2018Fairness,liu2021Just,zeng2022outlier}, but it remains to be seen whether these techniques can be transferred to SSL.

\subsubsection*{Acknowledgments}
This paper is based upon work supported by the National Science Foundation (NSF) under grants no.\ 2027737 and 2113373.

\bibliography{YK,sm}
\bibliographystyle{iclr2024_conference}

\appendix
\section{Supplementary details for under-representation study with \textsc{CIFAR10}}
\label{sec:cifar_supp}

\subsection{Supplementary details for SimCLR}
\label{sec:simclr-cifar-supp}

We use the implementation in \texttt{simclr.py} file of \href{https://github.com/p3i0t/SimCLR-CIFAR10}{https://github.com/p3i0t/SimCLR-CIFAR10} for the training of contrastive learning (CL) models with the SimCLR training protocol. Please see \texttt{simclr.py} in supplementary codes for parameter values. We use the same parameters in both training cases with balanced and imbalanced datasets. 

\subsubsection{representation and allocation harm} 

\label{sec:simclr-RH-AH-supp}
The representation and allocation harms for all 10 classes are provided in Figures \ref{fig:RH-100-fold-full} and \ref{fig:AH-100-fold-full}. Here, each row corresponds to an underrepresentation to  $1\%$ for the corresponding class. Additionally, in Figures \ref{fig:RH-20-fold-full} and \ref{fig:AH-20-fold-full} we present similar plots when classes are underrepresented at $5\%$ of their original sizes.

\begin{figure}[h]
     \centering
     \begin{subfigure}[b]{0.32\textwidth}
         \centering
         \includegraphics[width=\textwidth]{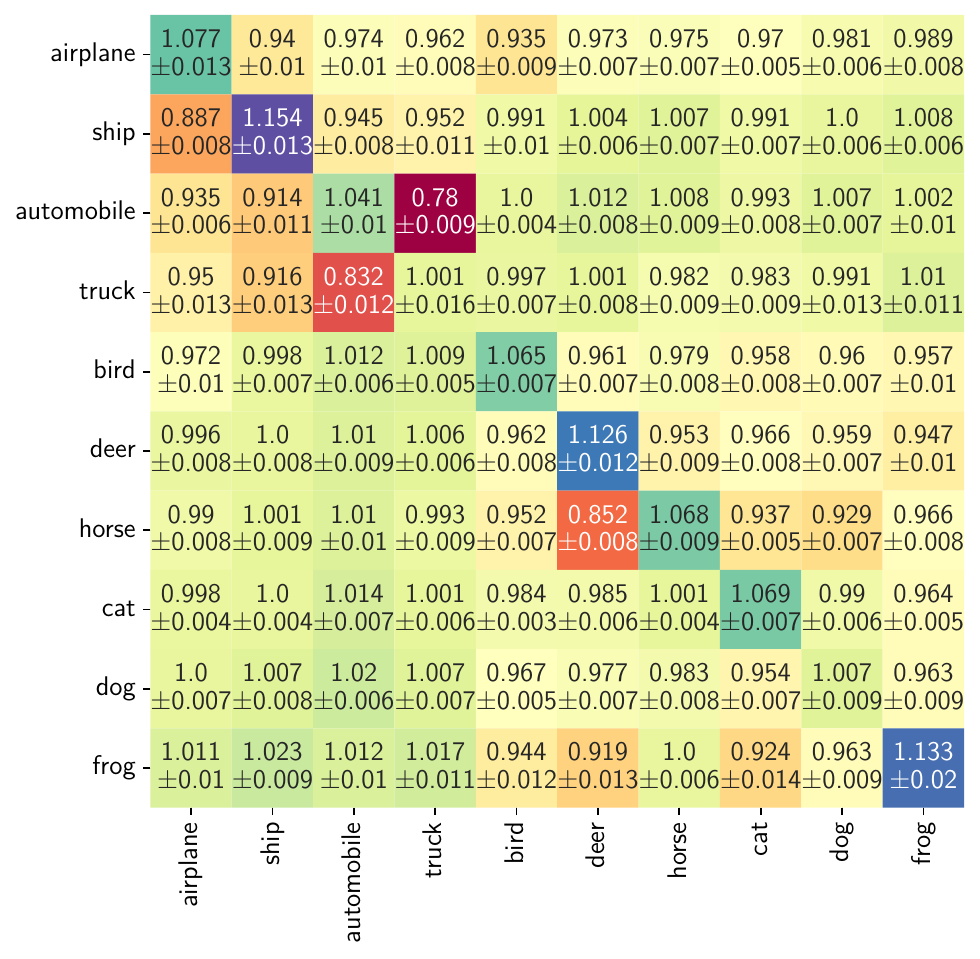}
         \caption{Representation harm}
         \label{fig:RH-100-fold-full}
     \end{subfigure}
     \hfill
     \begin{subfigure}[b]{0.66\textwidth}
         \centering
         \includegraphics[width=\textwidth]{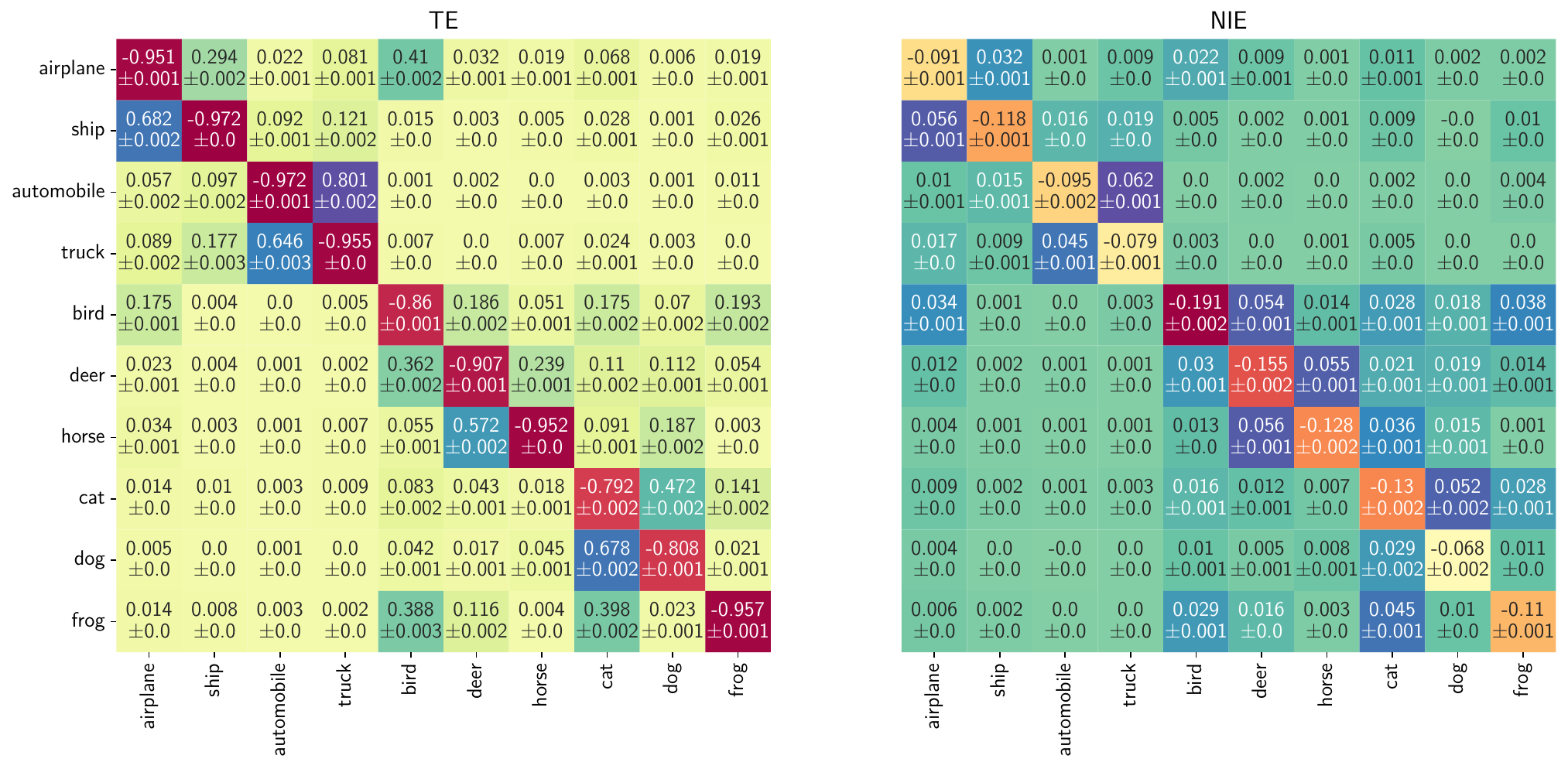}
         \caption{Allocation harm}
         \label{fig:AH-100-fold-full}
     \end{subfigure}
        \caption{Representation harm and total effect and natural indirect effect for allocation harm for SimCLR representations in \cifar\, where the classes are undersampled to $1\%$.}
        \label{fig:RH_AH_simclr_full_100_fold}
\end{figure}

\begin{figure}[h]
     \centering
     \begin{subfigure}[b]{0.32\textwidth}
         \centering
         \includegraphics[width=\textwidth]{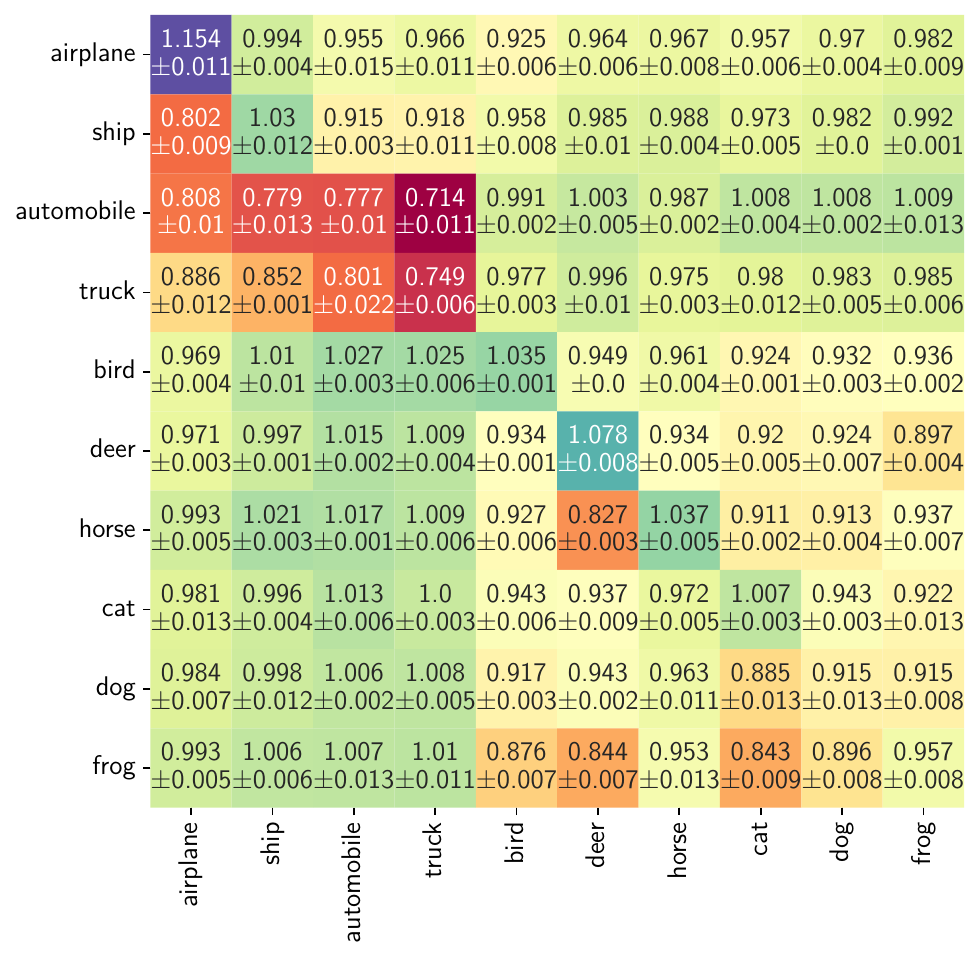}
         \caption{Representation harm}
         \label{fig:RH-20-fold-full}
     \end{subfigure}
     \hfill
     \begin{subfigure}[b]{0.66\textwidth}
         \centering
         \includegraphics[width=\textwidth]{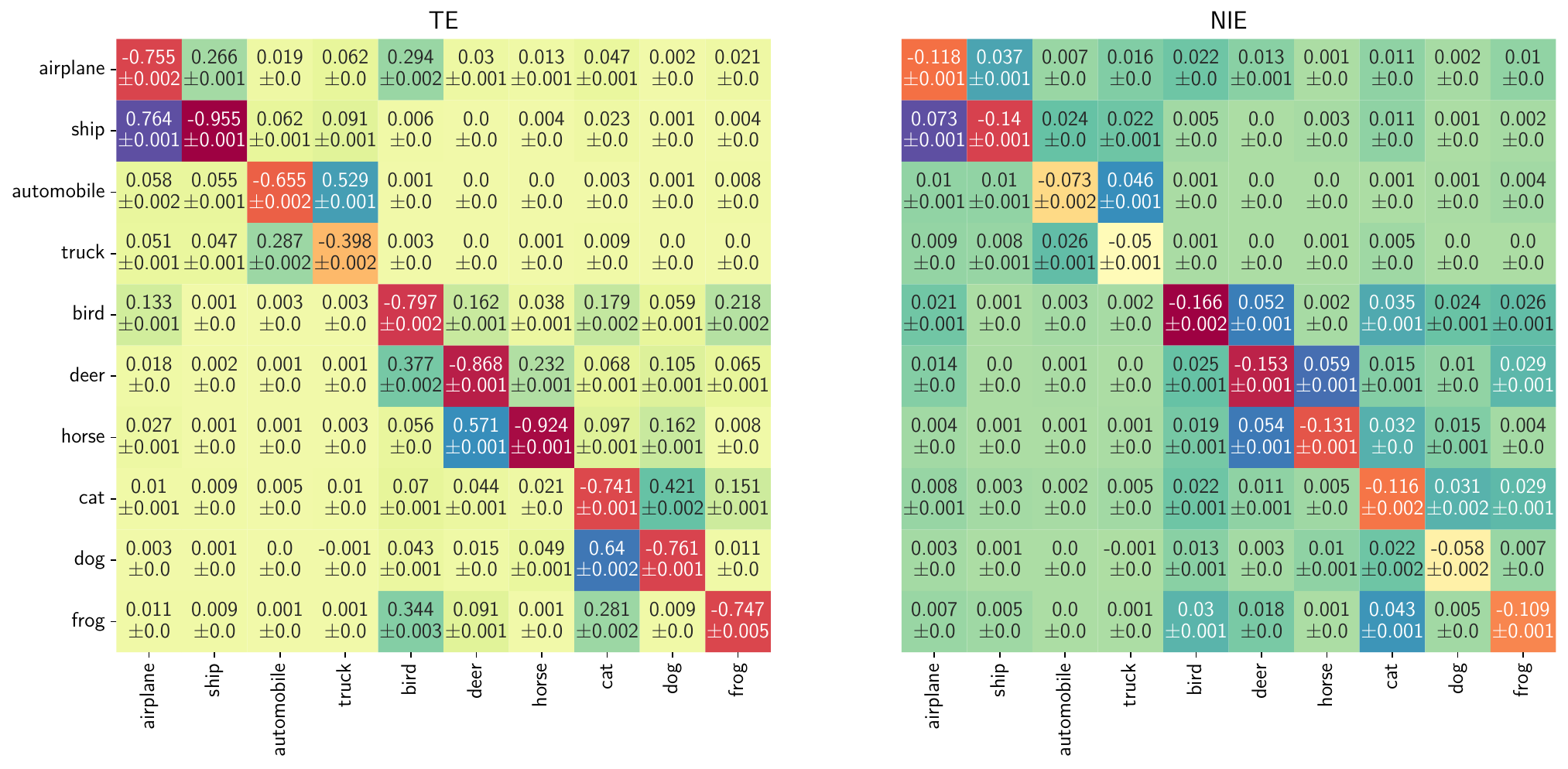}
         \caption{Allocation harm}
         \label{fig:AH-20-fold-full}
     \end{subfigure}
        \caption{Representation harm and total effect and natural indirect effect for allocation harm for SimCLR representations in \cifar\, where the classes are undersampled to $5\%$.}
        \label{fig:RH_AH_simclr_full_20_fold}
\end{figure}




\subsection{Supplementary details for SimSIAM} 
\label{sec:simsiam-cifar-supp}
We use the implementation in \texttt{main.py} file of  \href{https://github.com/Reza-Safdari/SimSiam-91.9-top1-acc-on-CIFAR10}{https://github.com/Reza-Safdari/SimSiam-91.9-top1-acc-on-CIFAR10} for the training of CL models with the SimSIAM protocol. Please see our \texttt{jobs.py} and \texttt{main.py} for the specification of hyperparameters, which are kept the same in both training cases with balanced and imbalanced datasets. We use 3 repetitions of each setting with three different seed values. The rest of the details are the same as in SimCLR training.

\begin{figure}[h]
     \centering
     \begin{subfigure}[b]{0.32\textwidth}
         \centering
         \includegraphics[width=\textwidth]{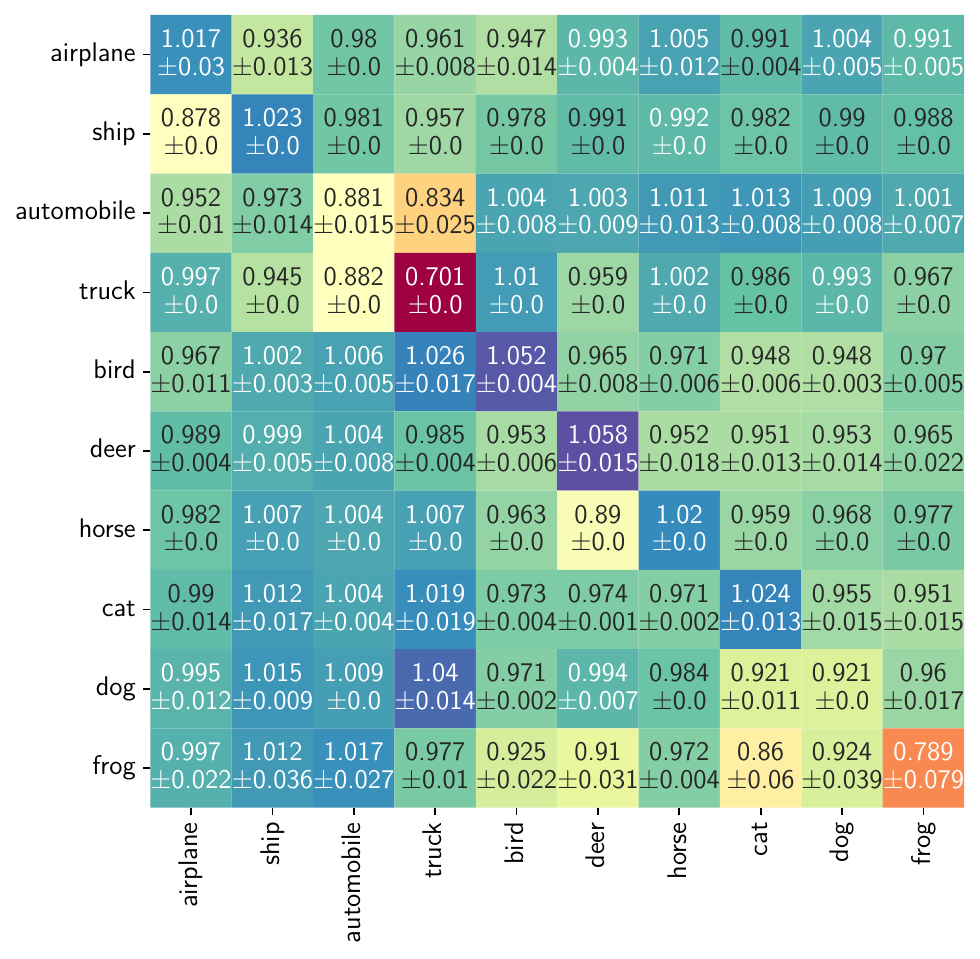}
         \caption{Representation harm}
         \label{fig:rep-harm-simsiam}
     \end{subfigure}
     \hfill
     \begin{subfigure}[b]{0.66\textwidth}
         \centering
         \includegraphics[width=\textwidth]{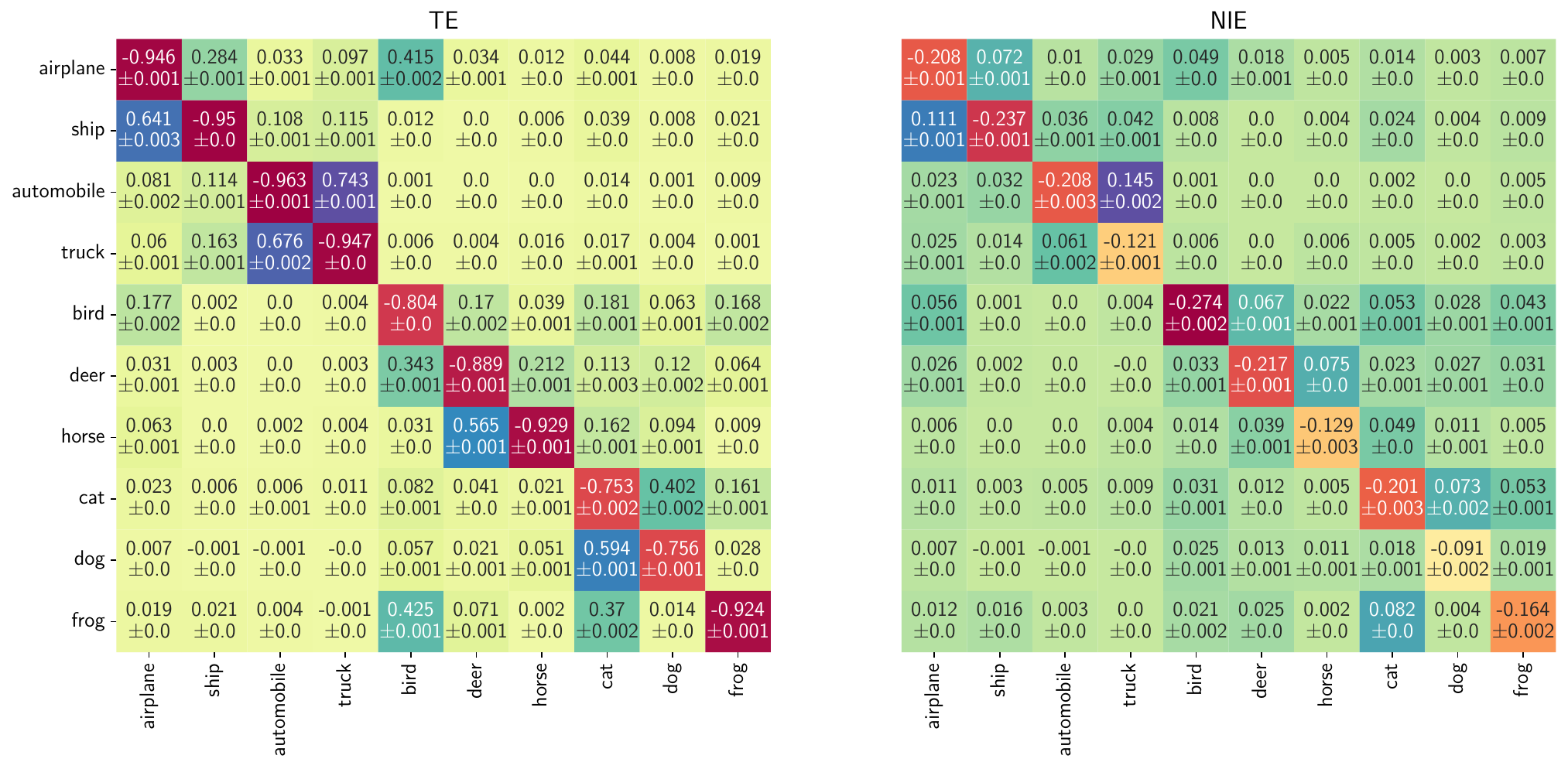}
         \caption{Allocation harm}
         \label{fig:allocation-harm-simsiam}
     \end{subfigure}
        \caption{Representation harm and total effect and natural indirect effect for allocation harm for SimSIAM representations in \cifar\ over three repetitions.}
        \label{fig:RH_AH_simsiam}
\end{figure}



\subsubsection{Representation and allocation harm}

\paragraph{Representation harm:} We plot the representation harm metrics \eqref{eq:rep-harm-metric} in Figure \ref{fig:rep-harm-simsiam}. Similar to SimCLR representations in Section \ref{sec:rep-harm}, representation harm is observed between semantically similar pairs (\texttt{automobiles}, \texttt{trucks}) (RH between them are $0.834 \pm 0.025$ and $0.882 \pm 0.0$), (\texttt{airplanes},  \texttt{ships}) (RH is $0.878 \pm 0.0$) and (\texttt{deer}, \texttt{horse}) (RH is $0.827 \pm 0.003$).

\paragraph{Allocation harm:} In Figure \ref{fig:allocation-harm-simsiam} we plot the TE and NIE metrics \eqref{eq:TE-NIE}, where we observe that representation harm 
causes allocation harm in a downstream classification. This is evident from the NIE metrics in Figure \ref{fig:allocation-harm-simsiam}, as the NIE 
between pairs  (\texttt{airplane}, \texttt{ship}) are $0.111 \pm 0.001$ and $0.072\pm 0.001$, between (\texttt{automobile}, \texttt{truck}) are $0.145 \pm 0.002$ and $0.061\pm 0.002$, and finally between (\texttt{deer}, \texttt{horse}) are $0.075 \pm 0.0$ and $0.039\pm 0.001$. Additionally, these classes suffer a non-trivial reduction in their prediction accuracies due to harm in representation. These observations are similar to our findings regarding allocation harm for SimCLR representation (Section \ref{sec:allocation-harm}).



\subsection{Boosted contrastive learning with SimCLR}

\label{sec:BCL}

We implement the boosted contrastive learning (BCL) algorithm \citep{zhou2022Contrastive} on SimCLR protocol using the implementation in \texttt{train.py} file in \href{https://github.com/MediaBrain-SJTU/BCL}{https://github.com/MediaBrain-SJTU/BCL}. Specifically, we use the \texttt{BCL\_I} version of the algorithm, whose hyperparameter values can be found in the files \texttt{jobs.py} and \texttt{train.py} in the \texttt{codes/cifar/BCL/} folder of our supplementary materials. Each setup is training for three repetitions with different seed values. 
The representation and allocation harm metrics for 10 \cifar\ classes are provided in Figure \ref{fig:RH_AH_bcl}. 
\begin{figure}[h]
     \centering
     \begin{subfigure}[b]{0.32\textwidth}
         \centering
         \includegraphics[width=\textwidth]{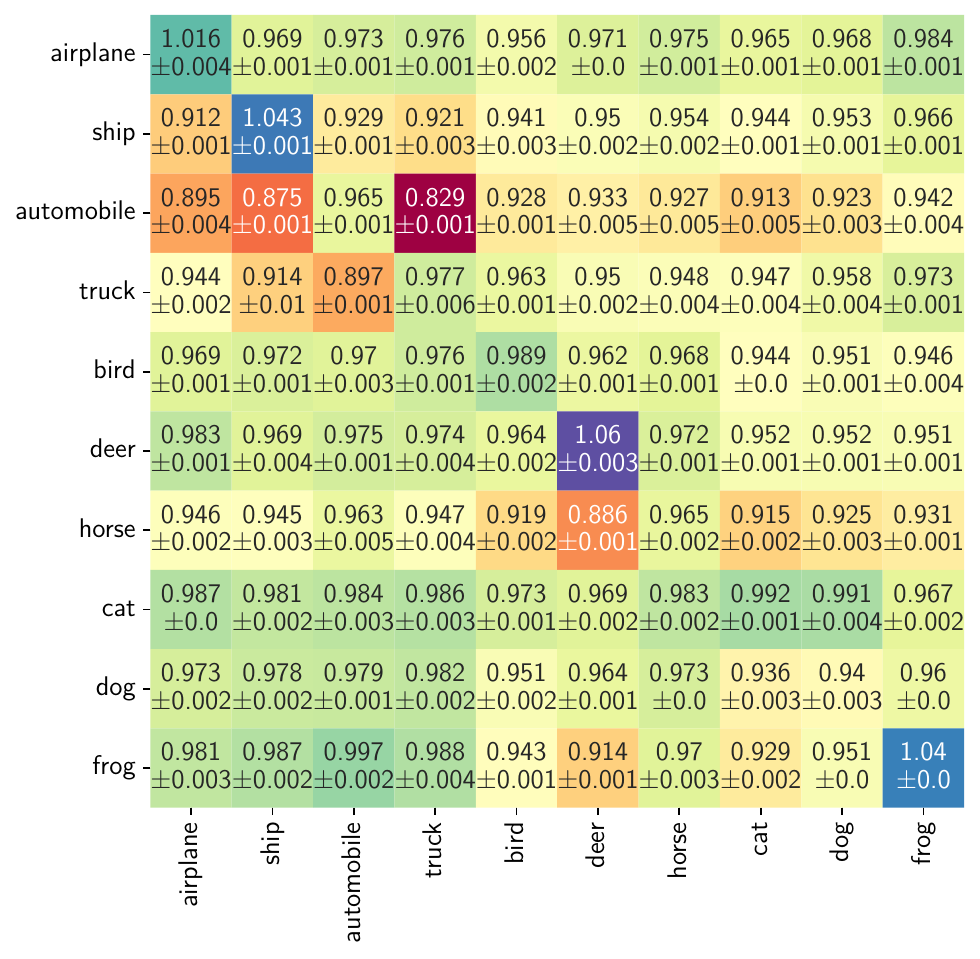}
         \caption{Representation harm}
         \label{fig:rep-harm-bcl}
     \end{subfigure}
     \hfill
     \begin{subfigure}[b]{0.66\textwidth}
         \centering
         \includegraphics[width=\textwidth]{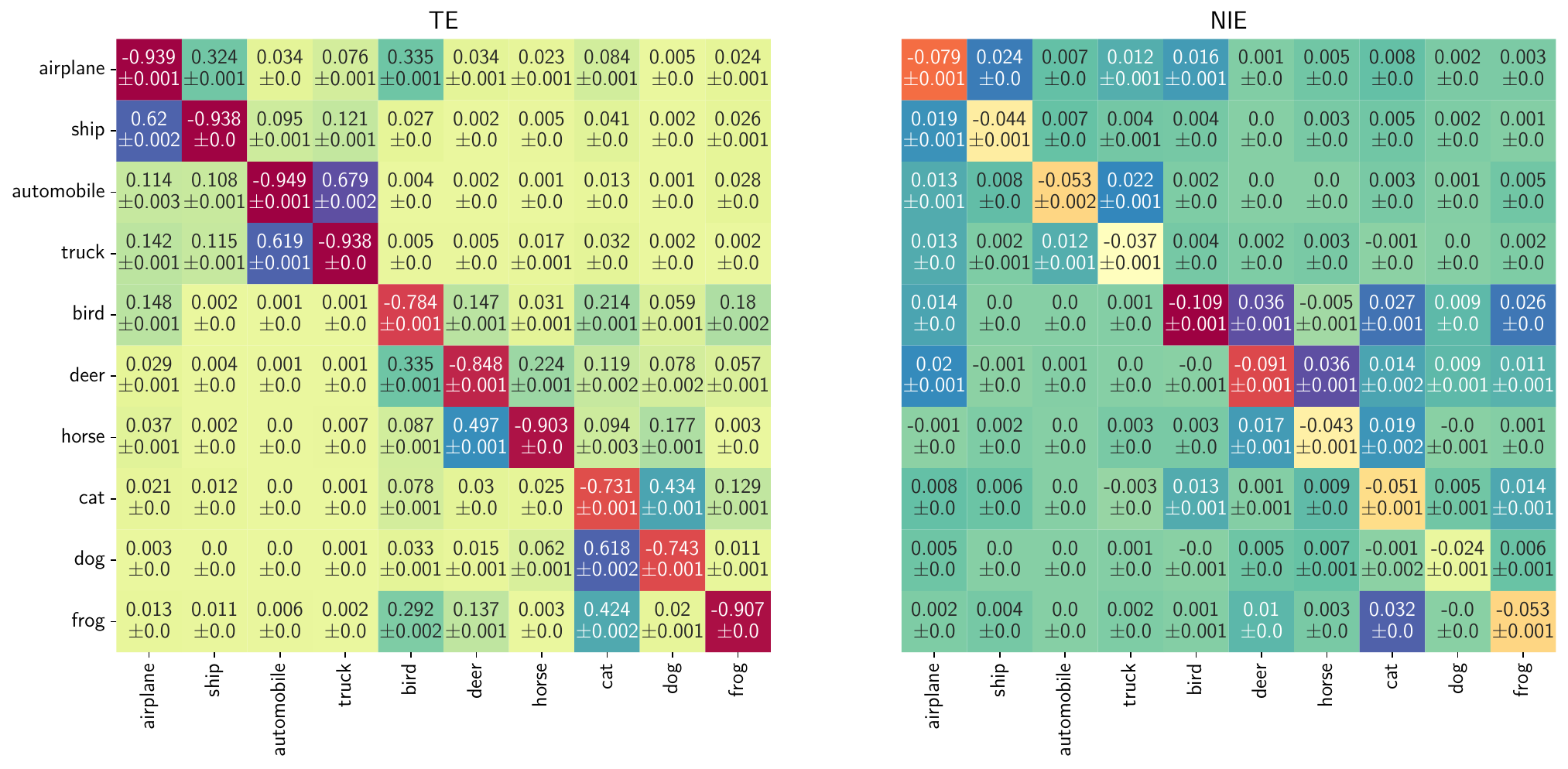}
         \caption{Allocation harm}
         \label{fig:allocation-harm-bcl}
     \end{subfigure}
        \caption{Representation harm and total effect and natural indirect effect for allocation harm for BCL representations in \cifar\ over three repetitions.}
        \label{fig:RH_AH_bcl}
\end{figure}

\section{Supplementary details for under-representation study with \textsc{BiasBios}}
\label{sec:bios_supp}

\subsection{\bios\ dataset details}
\label{subsec:bios_data_details}

In Figure \ref{fig:bios_sample_counts}, we provide the sample counts and frequencies for the 28 occupations and 2 genders (Male and Female) along with Total counts. We note that the dataset is imbalanced both in the occupation dimension as well as the gender dimension within a particular occupation. Particularly, the \texttt{Professor} occupation is the most occurring, while the \texttt{Rapper} occupation is the least occurring. Within each occupation, the proportion of samples of the two genders also varies mimicking societal gender stereotypes. For example, the occupations \texttt{dietician, interior designer, model, nurse, etc.} are dominated by \texttt{Female} samples, while the occupations \texttt{composer, DJ, pastor, rapper, software engineer, surgeon, etc.} are dominated by \texttt{Male} samples.


\begin{figure}[h]
     \centering
     \begin{subfigure}[b]{\textwidth}
         \centering
         \includegraphics[width=\textwidth]{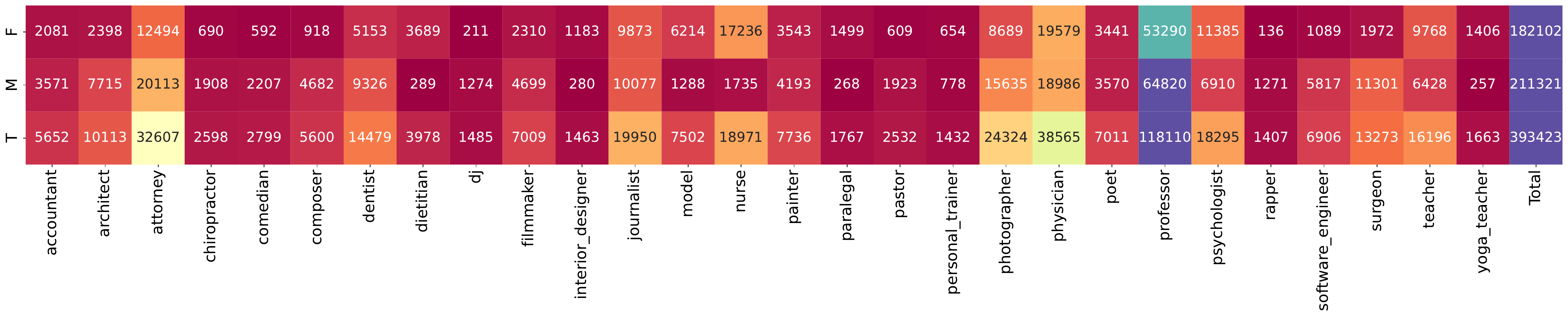}
     \end{subfigure}
     \hfill
     \begin{subfigure}[b]{\textwidth}
         \centering
         \includegraphics[width=\textwidth]{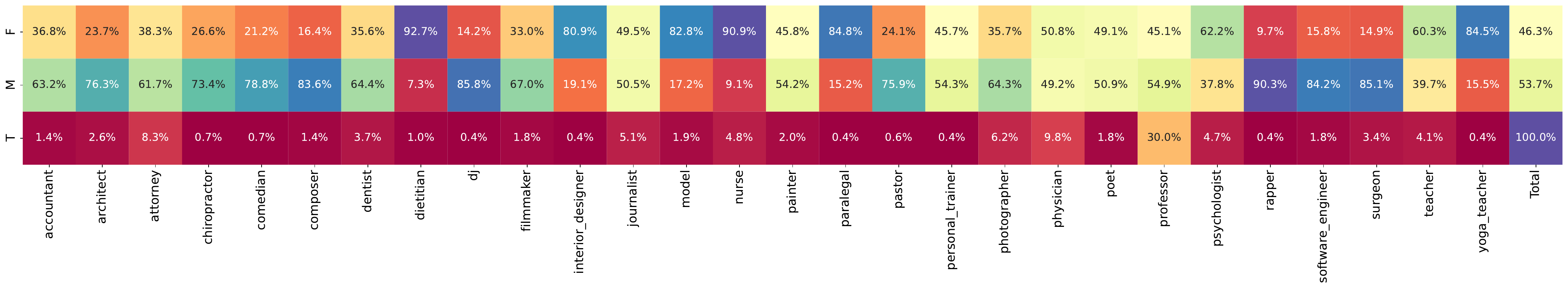}
     \end{subfigure}
     \caption{\bios\ dataset counts (above) and frequencies (below) by occupation and gender. The Y-axis labels are genders (M)ale, (F)emale, and (T)otal counts.}
     \label{fig:bios_sample_counts}
\end{figure}

\begin{wraptable}[9]{r}{0.35\textwidth}
\vspace{-1cm}
\caption{Training parameters for SimCSE.}\label{tab:simcse-parameters}
\begin{tabular}{cc}\\\toprule  
model parameters & value\\\midrule
batch size & $64$  \\
sequence length & $512$ \\
learning rate & $1e^{-5}$\\
training epochs & $1$ \\
\bottomrule
\end{tabular}
\vspace{-0.4cm}
\end{wraptable} 

\subsection{SimCSE experimental setup}
\label{subsec:simclr_exp_setup}

We randomly divide the $\sim$400k \bios\ dataset into the following three splits: 65\% as training set, 10\% as validation set, and 25\% as test set. We use the official SimCSE implementation\footnote{\url{https://github.com/princeton-nlp/SimCSE}} to train the embedding model, with the optimal parameters used in the reported experiments listed in Table \ref{tab:simcse-parameters}. To find these optimal parameters, we perform a grid search over the following parameters: training epochs $\in \{1, 3, 5\}$, learning rate $\in \{3e^{-5}, 1e^{-5}, 5e^{-5}\}$, batch size $\in \{64, 128, 256, 512\}$, and sequence lengths $\in \{32, 64, 128, 256, 512\}$, and then train a Logistic Regression model on top of each these embedding models. We select the parameters which result in the best occupation prediction accuracy on the validation set.

\subsection{Representation harm}
\label{subsec:grh_bios_all}

\begin{figure}
    \centering
    \vspace{-0.2in}
    \includegraphics[width=\linewidth]{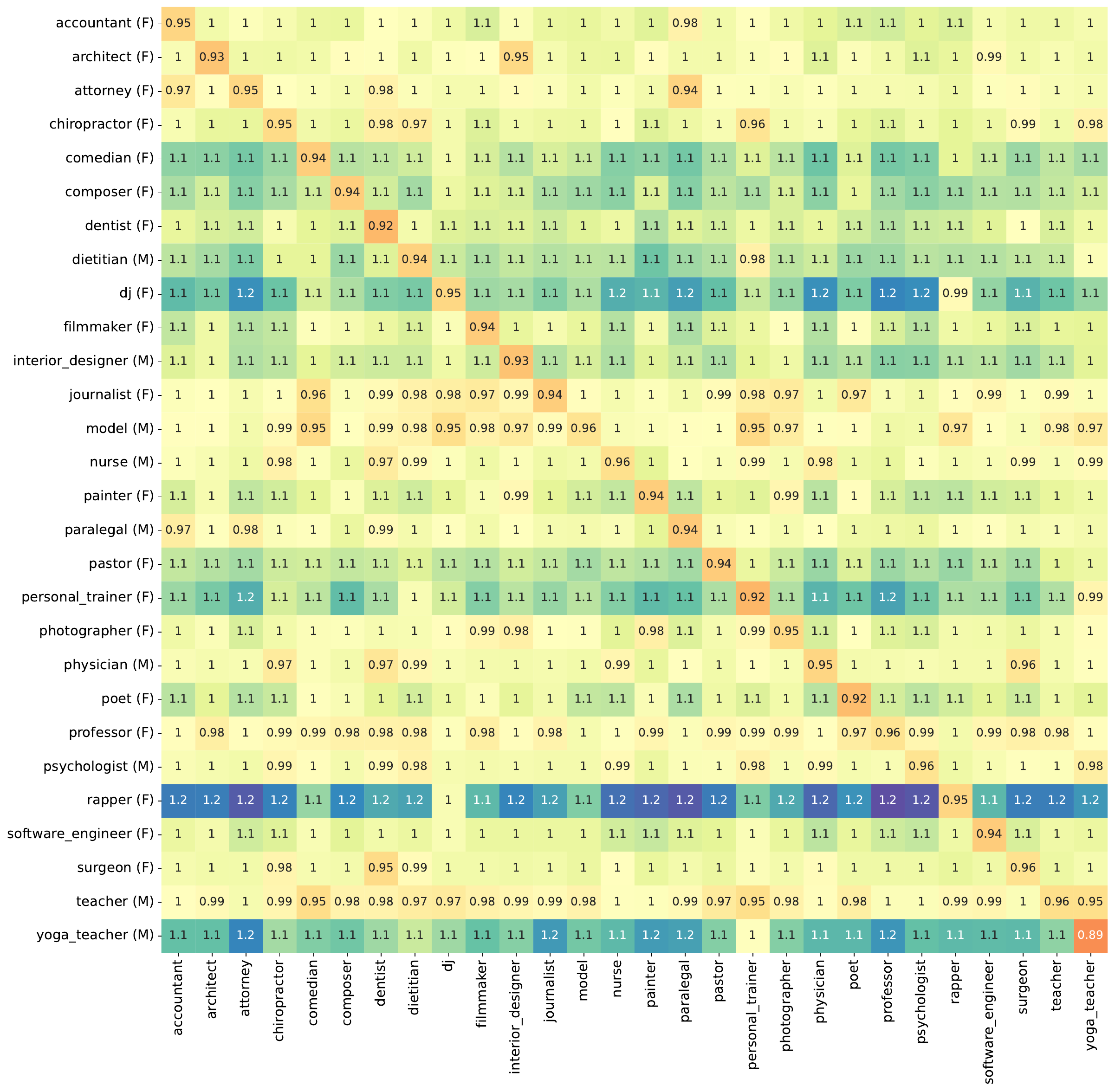}
    \vspace{-0.2in}
    \caption{Gender RH for all occupations in \bios\ dataset.}
    \label{fig:GRH_all_occps_bios}
\end{figure}


We present the GRH results for all occupation in the \bios\ dataset in Figure \ref{fig:GRH_all_occps_bios} (the F and M in the row labels indicate the under-represented gender for the corresponding occupation). We focus on the off-diagonal entries which should be >1 when there is no representation harm. We observe several deviations from this rule, broadly classifying into two types of deviations: (1) when representations of the under-represented group for an occupation collapse with a similar occupation, and (2) when representations of the under-represented group for an occupation collapses with several occupations. For the first type of deviation, we especially note \texttt{GRH(architect (F), interior\_designer) = 0.95}, \texttt{GRH(chiropractor (F), personal\_trainer) = 0.96}, \texttt{GRH(nurse (M), dentist) = 0.97}, \texttt{GRH(physician(M), surgeon) = 0.96}, and \texttt{GRH(surgeon (F), dentist) = 0.95} among others. These mentioned deviant groups are of occupations that are highly related. For the second type of deviation, we note that the \texttt{GRH} of Models (M), Journalists (F), and Teachers (M) have values $<1$ for many occupations in the dataset.

\subsection{Allocation harm}
\label{sec:AH-bios-supp}

We consider the task of occupation prediction from the SimCSE representations we learned on the \bios\ dataset using a logistic regression model trained on the same data as we used to learn the representations. Such a decision-making system may be used to assist in recruiting or hiring, applications where allocation harm can exacerbate gender disparity. To counteract the effect of under-representation at the supervised learning stage (thus focusing our attention on the effect of under-representation on CL), we reweigh the samples to achieve gender parity within each occupation when training the logistic regression following prior work \citep{idrissi2022simple,sagawa2020Investigation}. See Appendix \ref{subsec:bios_allocation_harm} for results with other weighting strategies.

\paragraph{Metric:} We define gender allocation harm (GAH) similarly to Equalized Odds, a popular group fairness criteria in the algorithmic fairness literature \citep{hardt2016Equality}:
\begin{equation}
\label{eq:allocation-harm-bios}
\text{GAH}(l,m) = P(\hat y = m\mid y = l, \text{female}) - P(\hat y = m \mid y = l, \text{male}).
\end{equation}
GAH can be understood as the difference of confusion matrices corresponding to each gender with entries far from 0 implying allocation harm. The average of the absolute values of the GAH diagonal is a common way to quantify violations of Equalized Odds.

\begin{wrapfigure}[21]{r}{0.55\linewidth}
\vspace{-0.27in}
    \centering
    \includegraphics[width=\linewidth]{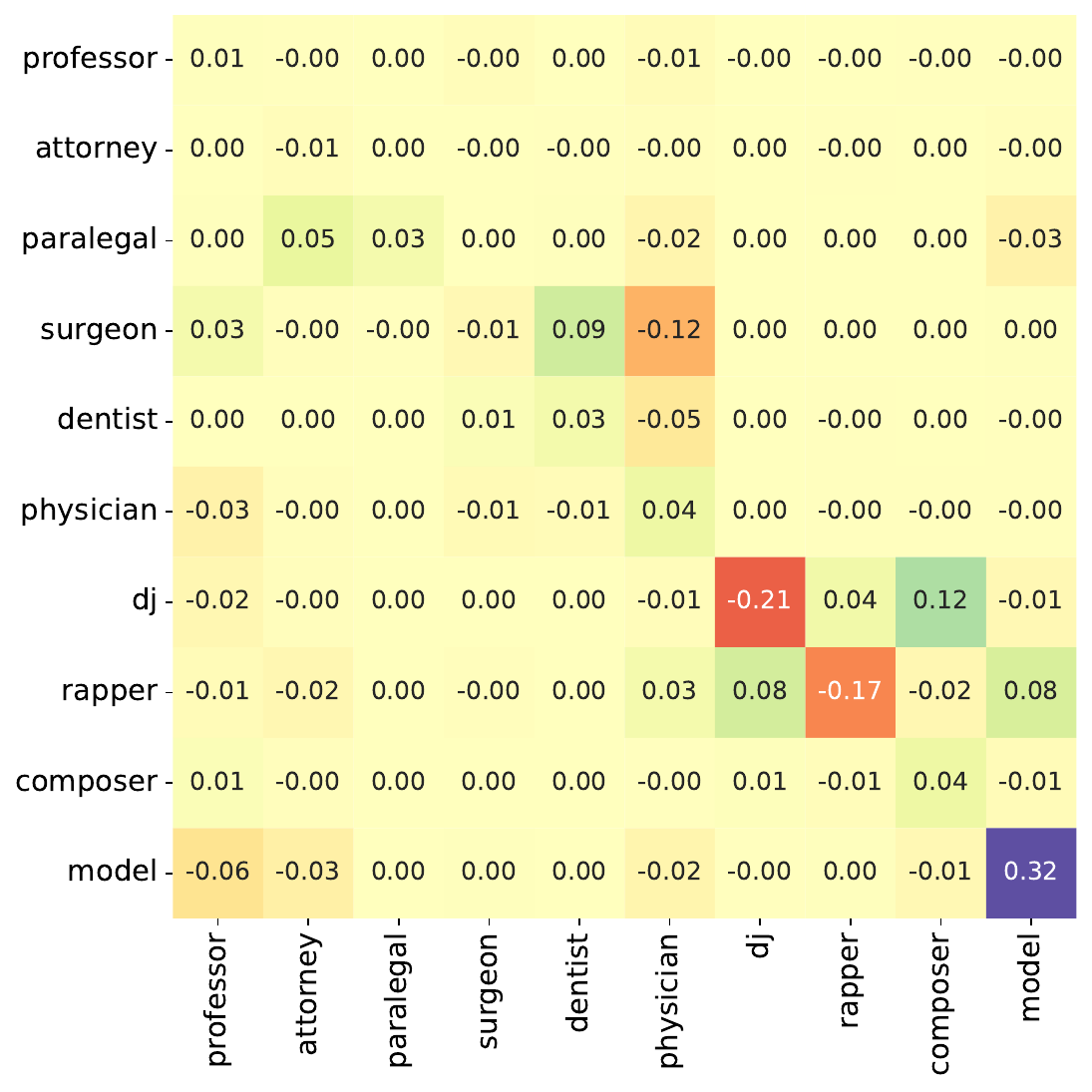}
    \vspace{-0.25in}
    \caption{Gender AH in \bios\ dataset.}
    \label{fig:allocation-harm-bios}
\end{wrapfigure}

\paragraph{Results:} In Figure \ref{fig:allocation-harm-bios} we present the GAH results for a subset of occupations. The complete set of GAH results is provided in Appendix \ref{subsec:bios_allocation_harm}. Inspecting the diagonal entries, we see large performance gaps between genders for occupations such as \texttt{rapper}, \texttt{DJ}, and \texttt{model} (within each of these occupations the majority gender corresponds to at least 85\% of the samples; these occupations are also underrepresented in the data with representation ranging from 0.4\% for \texttt{rapper} to 1.9\% for \texttt{model}). These results demonstrate that reweighing for gender parity when training the predictor may be insufficient to repair harms due to underrepresentation at the CL stage of the pipeline (it does, however, help to mitigate some of the biases of the unweighted model, as we show in the Appendix \ref{subsec:bios_allocation_harm}). We also note the mistake patterns for related occupations: female DJs are predicted as composers a lot more often than male ones, and female surgeons tend to be mistaken for dentists while male surgeons for physicians, despite similar performance on surgeons across genders. The \texttt{surgeon} example demonstrates that it may be insufficient to compare only class accuracies across genders (as is often done to measure group fairness violations via Equalized Odds) when quantifying allocations harms.

Compared to gender representation harms, we note that while the two have some overlap in terms of genders and occupations they are affecting, there are also differences, e.g., representations for female attorneys are closer to female paralegals than to male attorneys, but it does not manifest in the allocation harm analysis. We hypothesize that differences in which groups were affected by allocation and representation harms in our experiments might be due to the logistic regression model used for prediction utilizing only a subset (or subspace) of features most relevant to the task, while representation harm analysis takes into account all features.

The allocation harm does not have to be gender-specific. Similar to our \cifar\ case study, samples from an underrepresented occupation can also be mistaken for a related occupation at a similar rate between genders. We observe this for the \texttt{paralegal} occupation which corresponds to about 0.4\% of the training samples. Despite the gender imbalance (85\% of paralegals in the data are female), both male and female class accuracy is the worst across occupations (14\% for females and 11\% for males) and the majority of them (66\% for females and 61\% for males) is predicted as attorneys, which is a more frequent class (8.3\% of the samples are attorneys; see Appendix \ref{sec:bios_supp} for extended allocation harm analysis for occupations).

Overall, similar to the \cifar experiment, underrepresentation leads to the allocation harms of mistaking underrepresented groups for a related group. In some cases, the related group may be the same for both genders (e.g., in the case of paralegals), but in others, it can differ across genders (e.g., in the case of surgeons and DJs). Both cases can cause allocation harms for people from underrepresented occupations, whereas the latter additionally exacerbates gender stereotypes.

\subsection{Other details for allocation harm}
\label{subsec:bios_allocation_harm}

We experiment with three weighting strategies to counter the gender imbalance in the \bios\ dataset, namely: (1) Each sample is equally weighted, (2) We balance for gender imbalance within each class by weighting each sample as $W = \frac{N_y}{2 N_y^g N}$ where $N$ is the total number of samples in the dataset, $N_y$ is the number of samples within a class and $N_y^g$ is the number of samples of gender $g$ in class $y$, and (3) We balance for gender and class imbalance within the dataset by weighting each sample as $W = \frac{1}{2 * |y| * N_y^g}$ where $|y|$ is the number of classes and $N_y^g$ is the number of samples of gender $g$ in class $y$.

In Figure \ref{fig:bios_GAH_all} we provide the Gender Allocation Harms for all occupations in the \bios\ dataset and for the three weighting strategies. Similar to our observations in Section \ref{sec:bios}, inspecting the diagonal entries, we see large performance gaps between genders of the same occupation across all three weighting strategies. We do, however, see that reweighing for gender parity (Figure \ref{fig:bios_alloc_weightedLR}) does mitigate gender bias to a certain extent. For example, for the occupation \texttt{nurse} the \texttt{GAH} is mitigated, but it is merely reduced for \texttt{DJ} and \texttt{model}. Balancing for class frequencies in addition to gender (Figure \ref{fig:bios_alloc_weightedacross}) has little effect on the GAH.

In Figure \ref{fig:bios_AH} we present confusion matrices for the three weighting strategies to quantify allocation harms (AH) for occupations (irrespective of gender). For occupations such as \texttt{paralegal} and \texttt{interior designer} we observed a small amount of GAH, however, we see that performance on these occupations is poor for the unweighted and gender-balanced weighting strategies (Figures \ref{fig:bios_AH_stdLR} and \ref{fig:bios_AH_weightedLR}), as they are often confused with the related occupations (\texttt{attorney} and \texttt{architect} respectively). In this case, the representation harm is due to the under-representation of these occupations as opposed to the gender imbalance within occupations that we observed in the GAH experiments. Accounting for class imbalance in the weighting strategy helps to mitigate some of these biases (Figure \ref{fig:bios_AH_weightedacross}), e.g., the AH for \texttt{interior designer} is largely mitigated, while AH for \texttt{paralegal} is reduced.

Overall, we conclude that allocation harms due to under-representation in the contrastive learning stage can be \emph{partially} mitigated during the supervised learning stage by curating/reweighing the data to have equal representation of classes and groups.

\begin{figure}[h]
     \centering
     \begin{subfigure}[b]{0.32\textwidth}
         \centering
         \includegraphics[width=\textwidth]{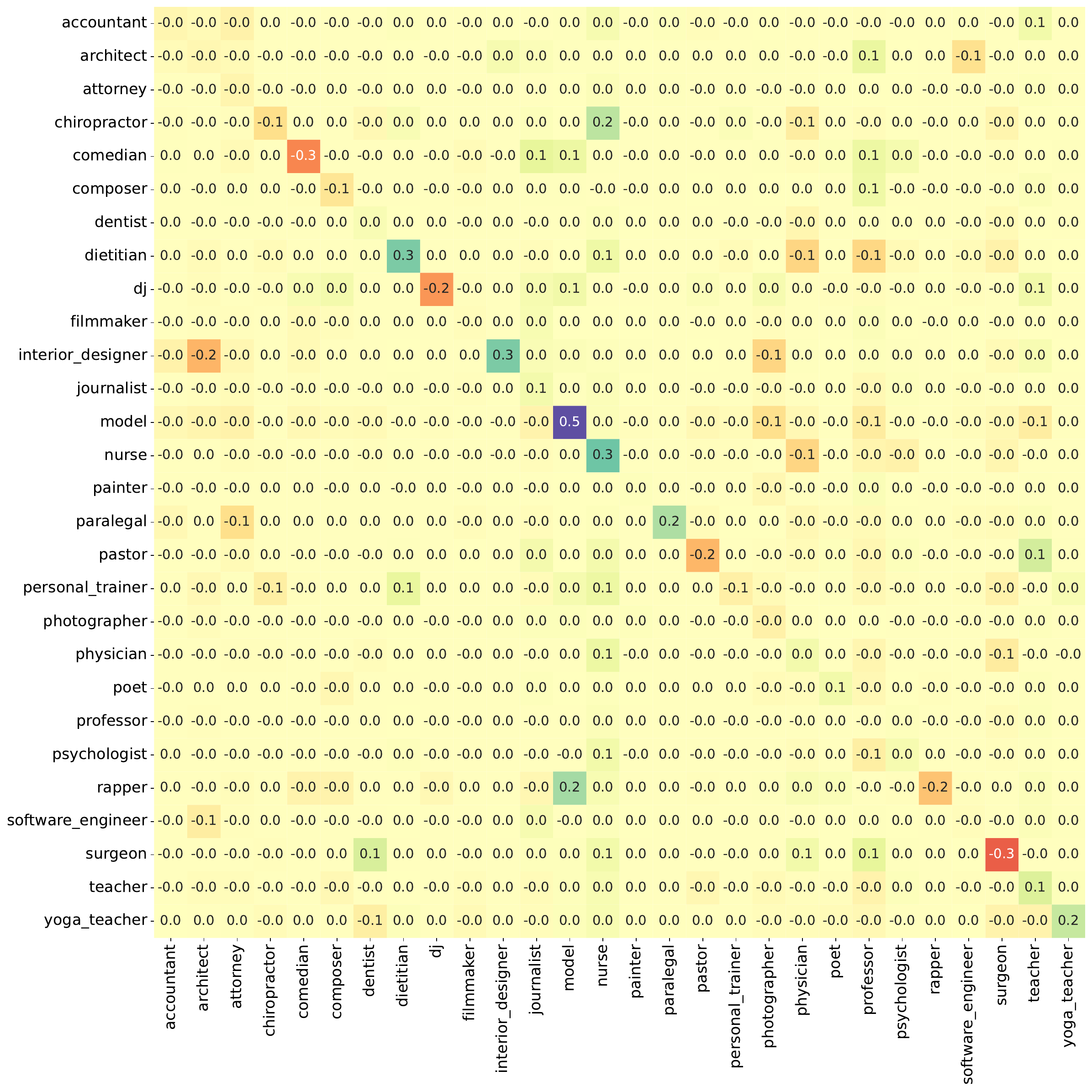}
         \caption{Equal weight for each sample}
         \label{fig:bios_alloc_stdLR}
     \end{subfigure}
     \hfill
     \begin{subfigure}[b]{0.32\textwidth}
         \centering
         \includegraphics[width=\textwidth]{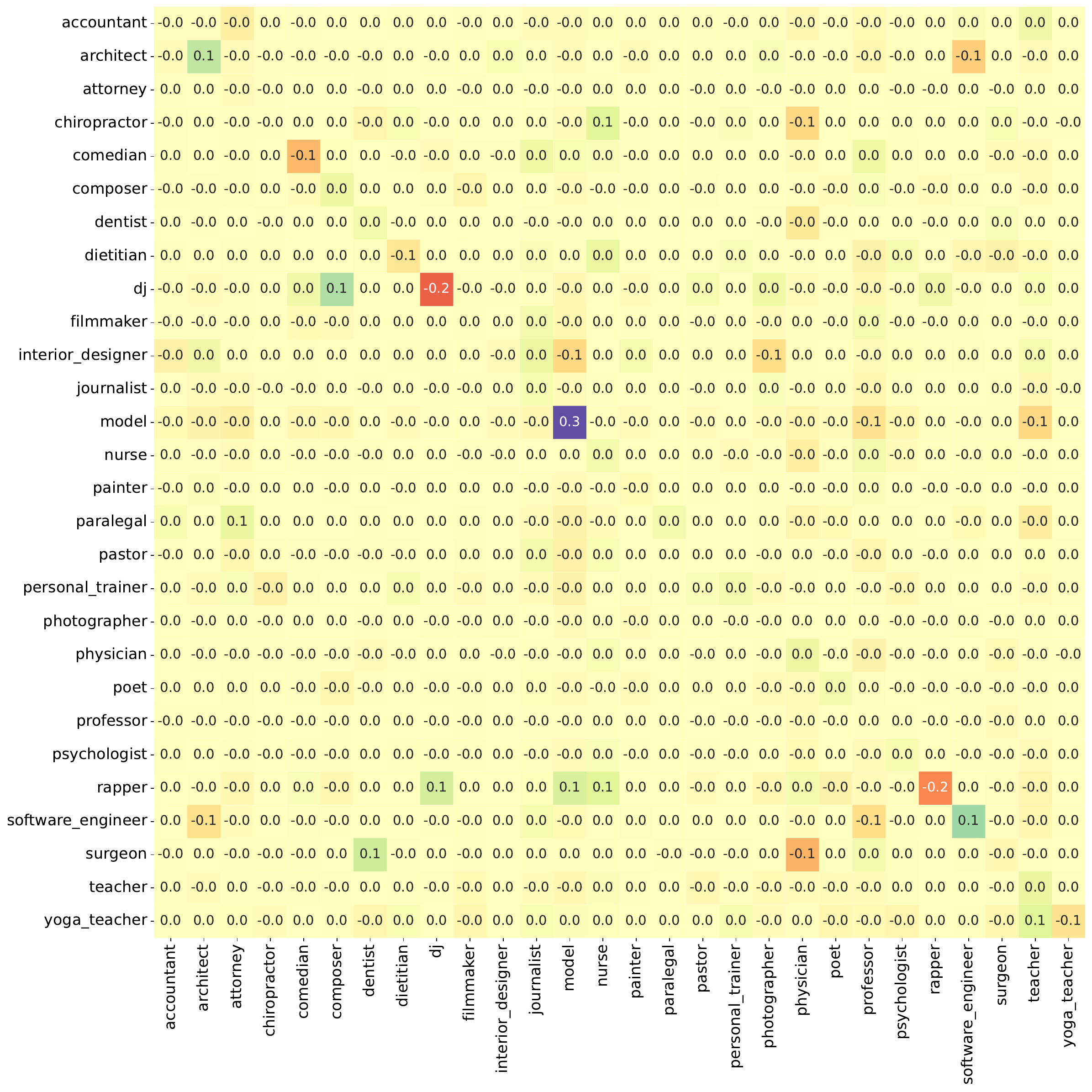}
         \caption{Gender-balanced}
         \label{fig:bios_alloc_weightedLR}
     \end{subfigure}
     \hfill
     \begin{subfigure}[b]{0.32\textwidth}
         \centering
         \includegraphics[width=\textwidth]{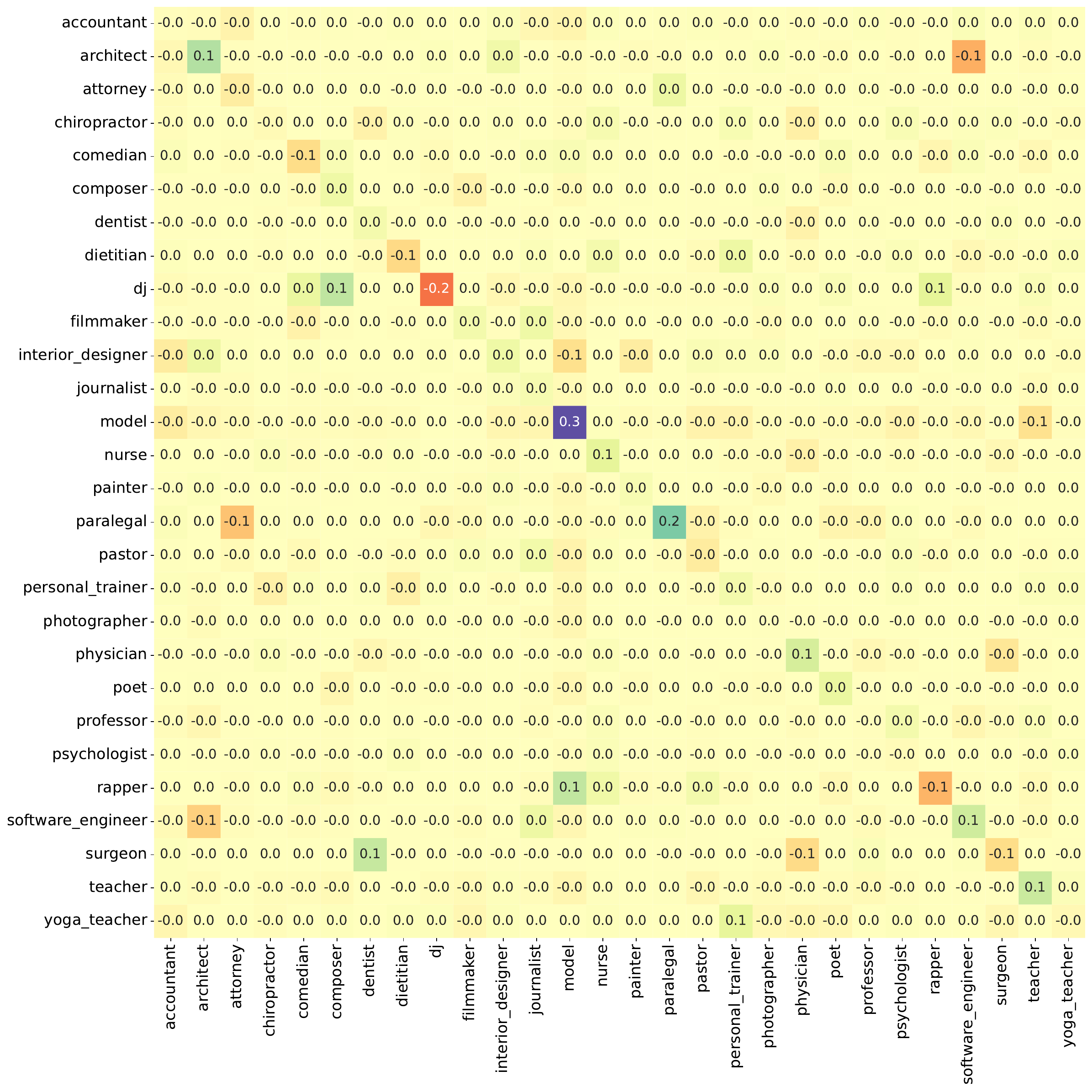}
         \caption{Gender and class-balanced}
         \label{fig:bios_alloc_weightedacross}
     \end{subfigure}
        \caption{Gender AH for all occupations in \bios\ across the three weighting strategies.}
        \label{fig:bios_GAH_all}
\end{figure}

\begin{figure}[h]
     \centering
     \begin{subfigure}[b]{0.32\textwidth}
         \centering
         \includegraphics[width=\textwidth]{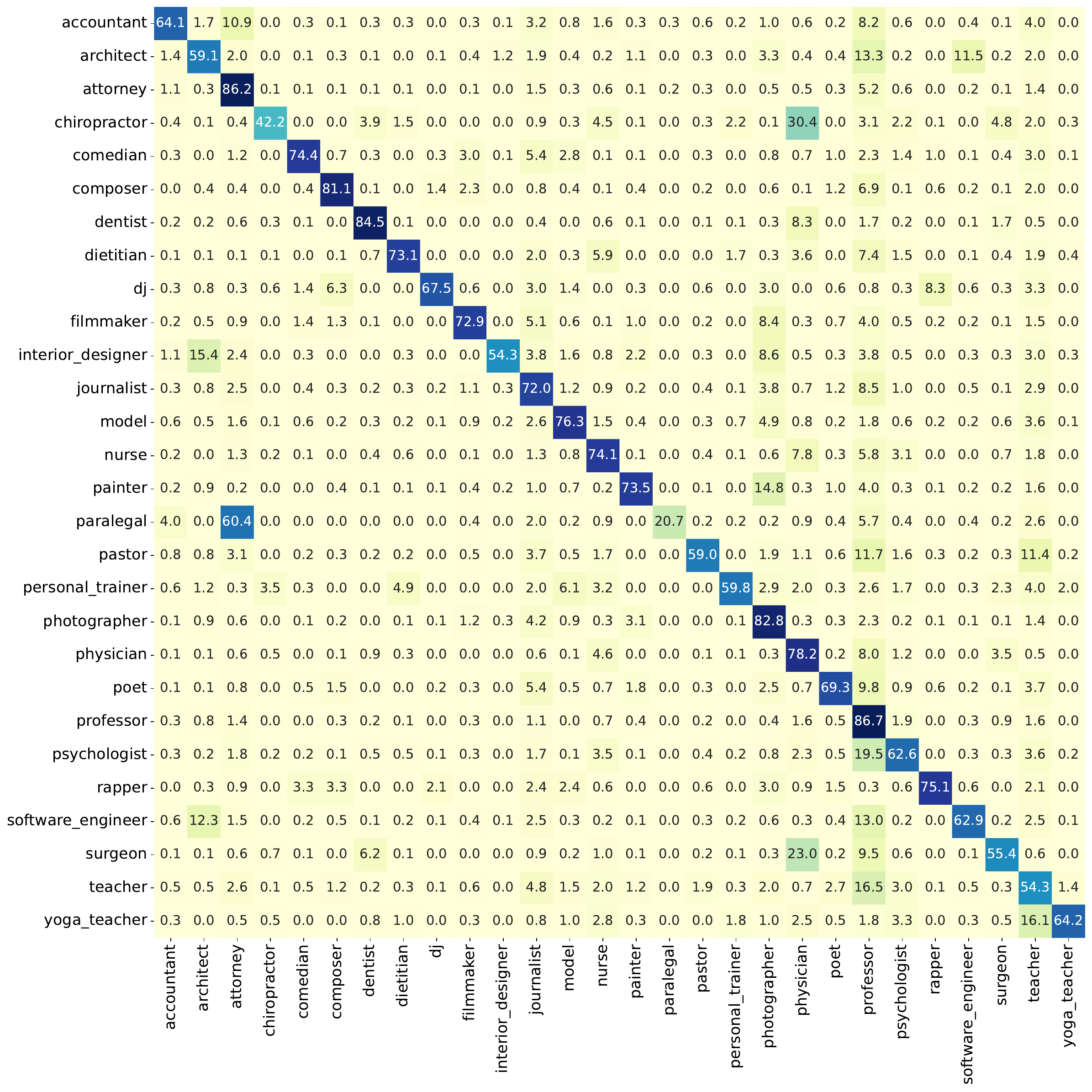}
         \caption{Equal weight for each sample}
         \label{fig:bios_AH_stdLR}
     \end{subfigure}
     \hfill
     \begin{subfigure}[b]{0.32\textwidth}
         \centering
         \includegraphics[width=\textwidth]{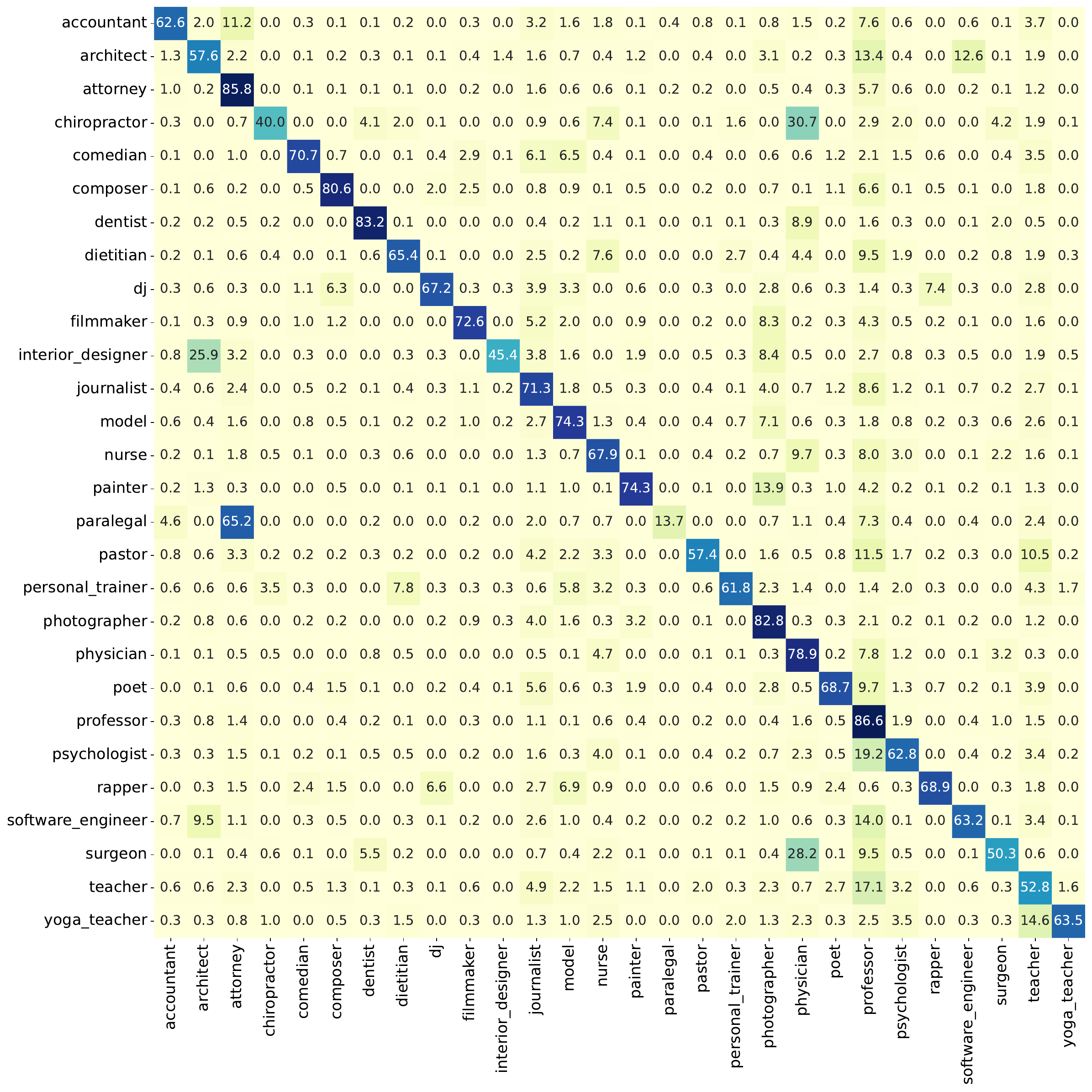}
         \caption{Gender-balanced}
         \label{fig:bios_AH_weightedLR}
     \end{subfigure}
     \hfill
     \begin{subfigure}[b]{0.32\textwidth}
         \centering
         \includegraphics[width=\textwidth]{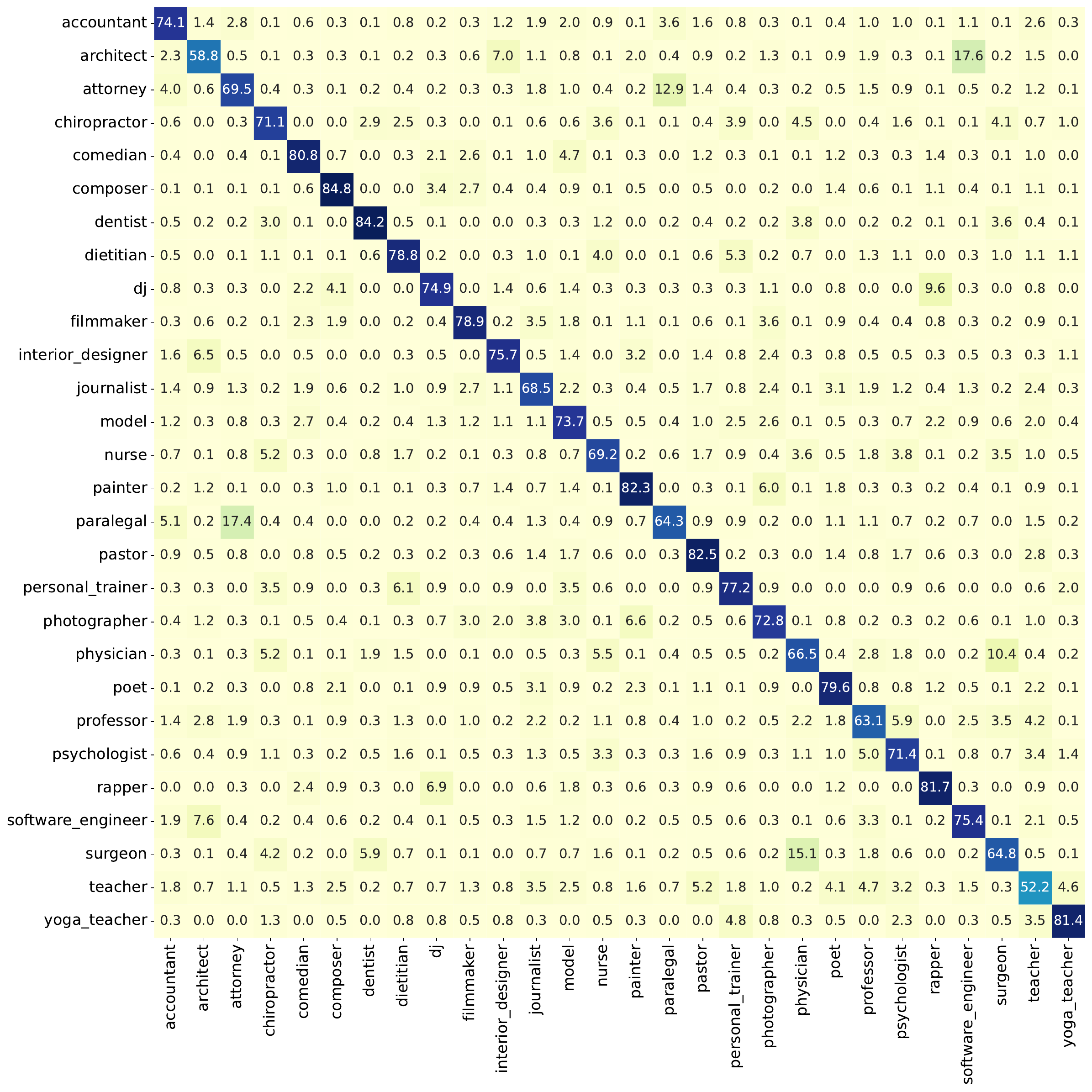}
         \caption{Gender and class-balanced}
         \label{fig:bios_AH_weightedacross}
     \end{subfigure}
        \caption{AH for all occupations in \bios\ across the three weighting strategies.}
        \label{fig:bios_AH}
\end{figure}

\section{A proof of Theorem \ref{th:layer-peeled-node2vec}}

We recall the CL loss \eqref{eq:node2vec-layer-peeled} and restate the theorem \ref{th:layer-peeled-node2vec} for the reader's convenience. 
The minimization problem in a layer-peeled setting is 
\begin{equation}
    \begin{aligned}
        \min\nolimits_{v_i \in \bS^{d-1}}\bL_\text{CL}(V), \\\textstyle
    \bL_{\text{CL}}(V) \triangleq - \frac1n \sum_{i\in [n]}  \frac{1}{\sum_{j\in[n]}e(i,j)} \sum_{j\in [n]} e(i, j) \log \Big\{\frac{\exp(\nicefrac{v_i^\top v_j}{\tau})}{\frac1n \sum_{l \in [n]}\exp(\nicefrac{v_i^\top v_l}{\tau})}\Big\}
    \end{aligned}
\end{equation}
where $\bS^{d-1}$ is the unit sphere in $\reals^d$. 
\begin{theorem}
    At $n\to \infty$ the optimum representations obtained from the minimization of CL loss in \eqref{eq:node2vec-layer-peeled} satisfy the following: $v_i^\star \stackrel{a.s.}{=} h_{Y_i}^\star$,
    where $\stackrel{a.s.}{=}$ denotes almost sure equality and $\{h_k^\star\}_{k \in [K]}$ is a minimizer for \begin{equation}
    \textstyle  \underset{h_k\in \bS^{d-1}}{\min}  - 
    \sum_{k_1 = 1}^K \pi_{k_1}
    \frac{ 
    \sum_{k_2 = 1}^K \pi_{k_2}  \alpha_{k_1, k_2} \frac{ h_{k_1}^\top h_{k_2}}{\tau}
      }{
      \sum_{k_2 = 1}^K \pi_{k_2} \alpha_{k_1, k_2}
      } 
      +
      \sum_{k_1 = 1}^K \pi_{k_1}
     \log \Big \{
    \sum_{k_3 = 1}^K  \pi_{k_3}
    e^{
    \nicefrac{h_{k_1}^\top h_{k_3}}{\tau}
    }
    \Big\}\,.
    \end{equation}
Note that the objective in \eqref{eq:node2vec-neural-collapse} is a weighted version of the node2vec objective in \eqref{eq:node2vec-layer-peeled} applied to common group-wise representations $h_{Y_i}$. 
\end{theorem} 

For our convenience, we index the $i$-th node in $k$-th block as $(k, i)$ where $i \in [n_k]$ and $k \in [K]$. We start with a simplification of the CL loss. 

\subsection{A simplification of the loss}
We divide the loss into two parts: a linear and a logarithmic sum exponential part. 
\begin{equation}
    \begin{aligned}
     \bL_{\text{CL}}(V) 
     & \triangleq \textstyle
    \frac 1{n} 
    {\sum_{(k_1, i_1)}} 
    \frac1{d(k_1, i_1)} {\sum_{(k_2, i_2)}}
    - e_{(k_1, i_1), (k_2, i_2)} \log \bigg \{
    \frac{
    \exp(
    \frac{v_{k_1, i_1}^\top v_{k_2, i_2}}{\tau}
   )
    }{
    \frac 1{n} {\sum_{(k_3, i_3)}} 
    \exp(
    \frac{v_{k_1, i_1}^\top v_{k_3, i_3}}{\tau}
    )
    }  
    \bigg\}\\
    & =  - \underbrace{\textstyle\frac 1{n} 
    \sum_{(k_1, i_1)}
    \frac1{d(k_1, i_1)} 
    \sum_{(k_2, i_2)}
     e_{(k_1, i_1), (k_2, i_2)} \big\{\frac{v_{k_1, i_1}^\top v_{k_2, i_2}}{\tau}\big\}}_{\triangleq{\bL}_{\text{linear}}(V)}\\
    & ~~~~ +  \underbrace{\textstyle\frac 1{n} 
    \sum_{(k_1, i_1)}
    \frac1{d(k_1, i_1)} 
    \sum_{(k_2, i_2)}  \log \big \{
    \frac 1{n}{\sum_{(k_3, i_3)}}
    \exp\big(
    \frac{v_{k_1, i_1}^\top v_{k_3, i_3}}{\tau}
    \big)
    \big\}}_{\triangleq{ \bL}_{\text{LSE}}(V)}\,.
    \end{aligned}
\end{equation} 
 We simplify the log-sum exponential part as  
 \begin{equation}
     \begin{aligned}
         \bL_{\text{LSE}}(V) & = \textstyle 
         \frac 1{n} 
    {\sum_{(k_1, i_1)}}
    \frac1{d(k_1, i_1)} 
    {\sum_{(k_2, i_2)}}  \log \big \{
    \frac 1{n}{\sum_{(k_3, i_3)}}
    \exp\big(
    \frac{v_{k_1, i_1}^\top v_{k_3, i_3}}{\tau}
    \big)
    \big\}\\
    & = \textstyle 
         \frac 1{n} 
    {\sum_{(k_1, i_1)}}
     \log \big \{
    \frac 1{n}{\sum_{(k_3, i_3)}}
    \exp\big(
    \frac{v_{k_1, i_1}^\top v_{k_3, i_3}}{\tau}
    \big)
    \big\}\,. 
     \end{aligned}
 \end{equation}

For $k$-th group, we define $\hat \pi_k \triangleq \frac{n_k}{n}$ as the sample proportion  and $h_k \triangleq \frac1 {n_k} \sum_{i = 1}^{n_k} v_{k, i}$ as the sample average of the representation vectors.
We further define 
\[
\textstyle  \widetilde {\bL}_{\text{linear}}(V) \triangleq \frac1\tau\sum_{k\in [K]} \hat \pi_{k} \frac{\sum_{k'\in [K]} \hat \pi_{k'} \alpha_{k, k'} h_{k}^\top h_{k'}}{\sum_{k' \in [K]} \hat \pi_{k'} \alpha_{k, k'}}  
\] and a modified loss:
\begin{equation} \label{eq:node2vec-modified}
    \check\bL_{\text{CL}}(V) \triangleq-\widetilde\bL_{\text{linear}}(V) + \bL_{\text{LSE}}(V)\,.
\end{equation} 
According to Lemma \ref{lemma:l2-convergence} the linear part of the loss has the following $\ell_2$ -convergence: as $n\to \infty$
\begin{equation}
    \begin{aligned}
         \max_V \big|\bL_{\text{{linear}}}(V) -  \widetilde {\bL}_{\text{{linear}}}(V)\big|  \stackrel{\textit{a.s.}}{\to} 0\,,
    \end{aligned}
\end{equation} and thus 
\begin{equation} \label{eq:node2vec-as-conv}
    \begin{aligned}
      & \max_{V}\big|\bL_{\text{CL}}(V) -  \widetilde {\bL}_{\text{CL}}(V)\big|\\
      &= \max_{V}\big|\bL_{\text{linear}}(V) + \bL_{\text{LSE}}(V) -  \widetilde {\bL}_{\text{linear}}(V) - \bL_{\text{LSE}}(V)\big|\\
      & = \max_{V}\big|\bL_{\text{linear}}(V) -  \widetilde {\bL}_{\text{linear}}(V)\big| \stackrel{\textit{a.s.}}{\to} 0\,.
    \end{aligned}
\end{equation}
In Sections \ref{sec:lb}, \ref{sec:eq-cond}, and \ref{sec:minimization} we perform a finite sample analysis on the modified CL loss in eq. \eqref{eq:node2vec-modified} and establish its neural collapse property. 



\subsection{A lower bound} \label{sec:lb}

For studying the minimization of modified CL loss in \eqref{eq:node2vec-modified} we derive an achievable lower bound for $\check\bL_{\text{CL}}(V)$. We expand the log-sum exponential part
\begin{equation}
    \begin{aligned}
     \bL_{\text{LSE}}(V) & = \textstyle   \frac 1{n} 
    {\sum_{(k_1, i_1)}}
      \log \big \{
    \frac 1{n}{\sum_{(k_3, i_3)}}
    \exp\big(
    \frac{v_{k_1, i_1}^\top v_{k_3, i_3}}{\tau}
    \big)
    \big\}  \\
    & = \textstyle\frac 1{n} 
    \sum_{(k_1, i_1)}
      \log \big \{
    \frac 1{n}{\sum_{(k_3, i_3)}}
    \exp\big(\frac{h_{k_1}^\top h_{k_3}}{\tau}\big)\exp\big(
    \frac{v_{k_1, i_1}^\top v_{k_3, i_3} -  h_{k_1}^\top h_{k_3}}{\tau}
    \big)
    \big\} \\
    & = \textstyle\frac 1{n} 
    \sum_{(k_1, i_1)}
      \log \bigg \{
    \frac{\frac 1{n}\sum_{(k_3, i_3)}
    \exp\big(\frac{h_{k_1}^\top h_{k_3}}{\tau}\big)\exp\big(
    \frac{v_{k_1, i_1}^\top v_{k_3, i_3} -  h_{k_1}^\top h_{k_3}}{\tau}
    \big)}{\frac 1{n}\sum_{(k_3, i_3)}
    \exp\big(\frac{h_{k_1}^\top h_{k_3}}{\tau}\big)}
    \bigg\}  \\
    & ~~~~+ \textstyle \frac 1{n} 
    \sum_{(k_1, i_1)}
      \log \big \{
    \frac 1{n}\sum_{(k_3, i_3)}
    \exp\big(\frac{h_{k_1}^\top h_{k_3}}{\tau}\big)
    \big\}\\
    & = \textstyle\frac 1{n} 
    \sum_{(k_1, i_1)}
      \log \bigg \{
    \frac{\frac 1{n}\sum_{(k_3, i_3)}
    \exp\big(\frac{h_{k_1}^\top h_{k_3}}{\tau}\big)\exp\big(
    \frac{v_{k_1, i_1}^\top v_{k_3, i_3} -  h_{k_1}^\top h_{k_3}}{\tau}
    \big)}{\sum_{k_3} \hat \pi_{k_3}
    \exp\big(\frac{h_{k_1}^\top h_{k_3}}{\tau}\big)}
    \bigg\}\\
    & ~~~~+ \textstyle\sum_{k_1} \hat \pi_{k_1}
      \log \big \{
    \sum_{k_3\in [K]} \hat \pi_{k_3}
    \exp\big(\frac{h_{k_1}^\top h_{k_3}}{\tau}\big)
    \big\}\\
    \end{aligned}
\end{equation} and provide a Jensen's inequality-based lower bound for 
\[
\textstyle\frac 1{n} 
    \sum_{(k_1, i_1)}
      \log \bigg \{
    \frac{\frac 1{n}\sum_{(k_3, i_3)}
    \exp\big(\frac{h_{k_1}^\top h_{k_3}}{\tau}\big)\exp\big(
    \frac{v_{k_1, i_1}^\top v_{k_3, i_3} -  h_{k_1}^\top h_{k_3}}{\tau}
    \big)}{\sum_{k_3} \hat \pi_{k_3}
    \exp\big(\frac{h_{k_1}^\top h_{k_3}}{\tau}\big)}
    \bigg\}\,.
\] Since $\log$ is a strictly concave function, we apply Jensen's inequality that $\log (\Ex[X]) \ge \Ex[\log X]$ to each summand of $(k_1, i_1)$ with respect to the probability 
 measure
\[
\textstyle \frac 1{n}{\sum_{(k_3, i_3)}} \frac{
    \exp(\nicefrac{h_{k_1}^\top h_{k_3}}{\tau}) }{\sum_{k'}\hat \pi_{k'}
    \exp(\nicefrac{h_{k_1}^\top h_{k'}}{\tau})} \delta_{k_3, i_3}\]
where $\delta_a$ is the Dirac-delta measure at $a$, (\ie\ $X = \exp\big(\frac{h_{k_1}^\top h_{k_3}}{\tau}\big)$)
and obtain 
\begin{equation} \label{eq:jensen}
    \begin{aligned}
        &  
      \textstyle
      \log \bigg \{
    \frac{\frac 1{n}\sum_{(k_3, i_3)}
    e^{\nicefrac{h_{k_1}^\top h_{k_3}}{\tau}}\exp\big(
    \frac{v_{k_1, i_1}^\top v_{k_3, i_3} -  h_{k_1}^\top h_{k_3}}{\tau}
    \big)}{\frac 1{n}\sum_{(k_3, i_3)}
    \exp\big(\frac{h_{k_1}^\top h_{k_3}}{\tau}\big)}
    \bigg\} \\
    & \ge \textstyle  \frac 1{n}{\sum_{(k_3, i_3)}} \Big\{\frac{
    \exp(\nicefrac{h_{k_1}^\top h_{k_3}}{\tau})}{\sum_{k_3}\hat \pi_{k_3}
    \exp(\nicefrac{h_{k_1}^\top h_{k_3}}{\tau})}\Big\}
     \Big(
    \frac{v_{k_1, i_1}^\top v_{k_3, i_3} -  h_{k_1}^\top h_{k_3}}{\tau}
    \Big) 
    \end{aligned}
\end{equation}
which we combine over $(k_1, i_1)$ to obtain the following
\begin{equation}
    \begin{aligned}
        &  
      \textstyle
     \frac1n \sum_{(k_1, i_1)} \log \bigg \{
    \frac{\frac 1{n}\sum_{(k_3, i_3)}
    e^{\nicefrac{h_{k_1}^\top h_{k_3}}{\tau}}\exp\big(
    \frac{v_{k_1, i_1}^\top v_{k_3, i_3} -  h_{k_1}^\top h_{k_3}}{\tau}
    \big)}{\frac 1{n}\sum_{(k_3, i_3)}
    e^{\nicefrac{h_{k_1}^\top h_{k_3}}{\tau}}}
    \bigg\} \\
    & \ge \textstyle \frac1n \sum_{k_1, i_1} \frac 1{n}{\sum_{k_3, i_3}} \Big\{\frac{
    \exp(\nicefrac{h_{k_1}^\top h_{k_3}}{\tau})}{\sum_{k_3}\hat \pi_{k_3}
    \exp(\nicefrac{h_{k_1}^\top h_{k_3}}{\tau})}\Big\}
     \Big(
    \frac{v_{k_1, i_1}^\top v_{k_3, i_3} -  h_{k_1}^\top h_{k_3}}{\tau}
    \Big)  \\
    & = \textstyle  \sum_{k_1} \hat \pi_{k_1} {\sum_{k_3}} \hat \pi_{k_3} \Big\{\frac{
    \exp(\nicefrac{h_{k_1}^\top h_{k_3}}{\tau})}{\sum_{k_3}\hat \pi_{k_3}
    \exp(\nicefrac{h_{k_1}^\top h_{k_3}}{\tau})}\Big\}
     \Big( 
    \frac{\frac1 {n_{k_1} } \sum_{i_1} \frac 1{n_{k_3}} \sum_{i_3}v_{k_1, i_1}^\top v_{k_3, i_3} -  h_{k_1}^\top h_{k_3}}{\tau}
    \Big)   = 0\,.
    \end{aligned}
\end{equation} The above simplification implies 
\begin{equation}
     \bL_{\text{LSE}}(V) \ge \textstyle\sum_{k_1} \hat \pi_{k_1}
      \log \big \{
    \sum_{k_3} \hat \pi_{k_3}
    \exp\big(\frac{h_{k_1}^\top h_{k_3}}{\tau}\big)
    \big\} \triangleq \widetilde \bL_{\text{LSE}}\,,
\end{equation} and thus 
\begin{equation} \label{eq:node2vec-modified-lb}
    \check\bL_{\text{CL}}(V) = -\widetilde\bL_{\text{linear}}(V) + \bL_{\text{LSE}}(V) \ge -\widetilde\bL_{\text{linear}}(V) + \widetilde \bL_{\text{LSE}}(V) \triangleq  \widetilde \bL_{\text{CL}}(V)\,.
\end{equation}

\subsection{Equality condition of lower bound} \label{sec:eq-cond}
Under what condition equality is achieved in the inequality \eqref{eq:node2vec-modified-lb}? To answer that we look at the Jensen's inequality applied in \eqref{eq:jensen}. Since $\log$ is a strictly concave function, equality is achieved when for any $(k_3, i_3)$ it holds
\begin{equation} \label{eq:jensen-eq}
    v_{k_1, i_1}^\top v_{k_3, i_3} -  {h_{k_1}}^\top {h_{k_3}} = c_{k_1, i_1}\,,
\end{equation} where $c_{k_1, i_1} \in \reals$ is the constant associated with the equality constraint for Jensen's inequality for $(k_1, i_1)$. Next, we establish that these constants $c_{k_1, i_1}$ must be exactly equal to zero. For this purpose, we average over both $i_1$ and $i_3$ in \eqref{eq:jensen-eq}
{\small \begin{equation}
    \begin{aligned}
        \textstyle\frac{1}{n_{k_1}} \sum_{i_1} \frac1 {n_{k_3}} \sum_{i_3} c_{k_1, i_1} & = \textstyle\frac{1}{n_{k_1}} \sum_{i_1 = 1}^{n_{k_1}} \frac{1}{n_{k_3}} \sum_{i_3 = 1}^{n_{k_3}}\{v_{k_1, i_1}^\top v_{k_3, i_3} -  {h_{k_1}}^\top {h_{k_3}}\}\\
        & = \textstyle {h_{k_1}}^\top {h_{k_3}} - {h_{k_1}}^\top {h_{k_3}} = 0\,.
    \end{aligned}
\end{equation}}
A further simplification of the above equation leads to 
\begin{equation}
    \textstyle\frac{1}{n_{k_1}} \sum_{i_1} c_{k_1, i_1}  = 0\,.
\end{equation}
Then, we let $k_1 = k_3 = k$ and $i_1 = i_3 = i$ in \eqref{eq:jensen-eq} and average over $i$ to obtain 
{\small \begin{equation}
    \begin{aligned}
        0 & =\textstyle\frac{1}{n_k} \sum_{i = 1}^{n_k} c_{k, i}\\
        & = \textstyle\frac{1}{n_k} \sum_{i} v_{k, i}^\top v_{k, i} -  {h_k}^\top {h_k}\\
        & = \textstyle\frac{1}{n_k} \sum_{i} (v_{k, i} - {h_k})^\top (v_{k, i} - {h_k}) = \frac1{n_k} \sum_i \|v_{k, i} - h_k\|_2^2
    \end{aligned}
\end{equation}} which finally concludes that at the equality 
\begin{equation} \label{eq:jensen-eq-final}
    v_{k, i} = h_k\,.
\end{equation} We further notice that $v_{k,i} \in \bS^{d-1}$ which implies $h_{k} \in \bS^{d-1}$.

\subsection{Minimization of eq. \eqref{eq:node2vec-modified}} \label{sec:minimization}
Further developing on  eq. \eqref{eq:node2vec-modified-lb} we provide a final lower bound for eq. \eqref{eq:node2vec-modified}. 
\begin{equation} \label{eq:node2vec-modified-lb-final}
    \begin{aligned}
        \check\bL_{\text{CL}}(V)  & \ge   \widetilde \bL_{\text{CL}}(V)\\
        & = \textstyle - \frac1\tau\sum_{k\in [K]} \hat \pi_{k} \frac{\sum_{k'\in [K]} \hat \pi_{k'} \alpha_{k, k'} h_{k}^\top h_{k'}}{\sum_{k' \in [K]} \hat \pi_{k'} \alpha_{k, k'}} + \sum_{k_1} \hat \pi_{k_1}
      \log \big \{
    \sum_{k_3} \hat \pi_{k_3}
    \exp\big(\frac{h_{k_1}^\top h_{k_3}}{\tau}\big)
    \big\} \\
    & \ge \textstyle  - \frac1\tau\sum_{k\in [K]} \hat \pi_{k} \frac{\sum_{k'\in [K]} \hat \pi_{k'} \alpha_{k, k'} {\hat h_{k}}^\top {\hat h_{k'}}}{\sum_{k' \in [K]} \hat \pi_{k'} \alpha_{k, k'}} + \sum_{k_1} \hat \pi_{k_1}
      \log \big \{
    \sum_{k_3} \hat \pi_{k_3}
    \exp\big(\frac{{\hat h_{k_1}}^\top {\hat h_{k_3}}}{\tau}\big)
    \big\}
    \end{aligned}
\end{equation} where 
\begin{equation} \label{eq:node2vec-neural-collapse-finite-sample}
\{\hat h_k\} \in \argmin_{h_k \in \bS^{d - 1}} \left \{\begin{aligned}
      &\textstyle- \frac1\tau\sum_{k\in [K]} \hat \pi_{k} \frac{\sum_{k'\in [K]} \hat \pi_{k'} \alpha_{k, k'} h_{k}^\top h_{k'}}{\sum_{k' \in [K]} \hat \pi_{k'} \alpha_{k, k'}}\\
    & \textstyle + \sum_{k_1} \hat \pi_{k_1}
      \log \big \{
    \sum_{k_3} \hat \pi_{k_3}
    \exp\big(\frac{h_{k_1}^\top h_{k_3}}{\tau}\big)\big\}
\end{aligned}\right\}\,.
\end{equation} The lower bound in eq. \eqref{eq:node2vec-modified-lb-final} does not involve any optimization variables ($v_{k, i}$ or $h_k$) and according to eq. \eqref{eq:jensen-eq-final} the equality achieved when 
\begin{equation}
    v_{k, i} = \hat h_k\,.
\end{equation} Thus, at the minimum of $\check \bL_{\text{CL}}(V)$ it holds $v_{k, i} = \hat h_k$. 

\subsection{Final conclusion of Theorem \ref{th:layer-peeled-node2vec}}
In Sections \ref{sec:lb}, \ref{sec:eq-cond}, and \ref{sec:minimization} we have established neural collapse for minimization of  $\check \bL_{\text{CL}}(V)$. It remains to see how that translates to the minimization of $ \bL_{\text{CL}}(V)$. In eq. \eqref{eq:node2vec-as-conv} we argued that \begin{equation} 
       \max_{V}\big|\bL_{\text{CL}}(V) -  \widetilde {\bL}_{\text{CL}}(V)\big|\stackrel{\textit{a.s.}}{\to} 0\,.
\end{equation} Since both $ \bL_{\text{CL}}(V)$ and $\check \bL_{\text{CL}}(V)$ are continuous, we use a continuous mapping theorem to conclude that
\begin{equation}
   d \big( \argmin_V \bL_{\text{CL}}(V) ,  \argmin_V \check\bL_{\text{CL}}(V) \big) \stackrel{\textit{a.s.}}{\to} 0
\end{equation} in any metric $d$ for measuring set difference, or, alternatively speaking 
\begin{equation}
    \argmin_V \bL_{\text{CL}}(V) \stackrel{\textit{a.s.}}{\to} \argmin_V \check\bL_{\text{CL}}(V) \,.
\end{equation} Since, the loss in eq. \eqref{eq:node2vec-neural-collapse-finite-sample} convergences almost surely to the loss in eq. \eqref{eq:node2vec-neural-collapse} we conclude the Theorem \ref{th:layer-peeled-node2vec}.

\subsection{Additional lemma}

\begin{lemma} \label{lemma:tech1-as}
    Consider a generic sequence of vectors $\{\xi_{k, i}\}_{k \in [K], i\in [n_k]}\subset \bS^{p-1} $ 
and define the group mean $\bar \xi _k = \frac1{n_k} \sum_i \xi_{k, i}$ and
\begin{equation}
    \eps(\Xi) \triangleq \textstyle \frac 1n \sum_{(k_2, i_2)}
     e_{(k_1, i_1), (k_2, i_2)}  \xi_{k_2, i_2}  - \sum_{k_2} \hat \pi_{k_2} \alpha_{k_1, k_2} \bar \xi_{k_2}\,.
\end{equation} Then 
\begin{equation}
    \textstyle \max_\Xi \|\eps(\Xi)\|_2 \stackrel{\textit{a.s.}}{\to}  0\,.
\end{equation}
\end{lemma}

\begin{proof}[Proof of the lemma \ref{lemma:tech1-as}] Indexing the dependence of $\eps(\Xi)$ as $\eps_n(\Xi)$ we notice that 
\begin{equation}
    \begin{aligned}
        \eps_n(\Xi) & = \textstyle \frac 1n \sum_{(k_2, i_2)}
     e_{(k_1, i_1), (k_2, i_2)}  \xi_{k_2, i_2}  - \sum_{k_2} \hat \pi_{k_2} \alpha_{k_1, k_2} \bar \xi_{k_2}\\
     & = \textstyle \frac 1n \sum_{(k_2, i_2)}
     \big\{e_{(k_1, i_1), (k_2, i_2)} - \alpha_{k_1, k_2}\big\}  \xi_{k_2, i_2}  
    \end{aligned}
\end{equation} and 
\begin{equation}
    \begin{aligned}
        \Ex[\eps_n(\Xi)^4] &= \textstyle  \Ex\big[ \big\{\frac 1n \sum_{(k_2, i_2)}
     \big\{e_{(k_1, i_1), (k_2, i_2)} - \alpha_{k_1, k_2}\big\}  \xi_{k_2, i_2} \big\}^4\big]\\
     & = \textstyle\frac1{n^4} \sum_{(k_2, i_2)} \Ex[\{e_{(k_1, i_1), (k_2, i_2)} - \alpha_{k_1, k_2}\}^4] \\
     & ~~~~ + \textstyle\frac1{n^4} \sum_{(k_2, i_2) \neq (k_3, i_3)} \Ex[\{e_{(k_1, i_1), (k_2, i_2)} - \alpha_{k_1, k_2}\}^2]\Ex[\{e_{(k_1, i_1), (k_3, i_3)} - \alpha_{k_1, k_3}\}^2] \\
     & = \textstyle O(\frac{1}{n^3}) + O(\frac{1}{n^2}) = O(\frac1{n^2})\,.
    \end{aligned}
\end{equation} Thus, for any $\delta>0$ we have 
\begin{equation}
    \begin{aligned}
      & \textstyle   \delta^4 \sum_{n \ge 1} P \big( \|\eps_n(\Xi) \|_2 > \delta \big)  \le \textstyle \sum_{n \ge 1}\Ex[\eps_n(\Xi)^4] < \infty\,.
    \end{aligned}
\end{equation}
Next, we use the first Borel-Cantelli lemma to conclude that for any $\Xi$ we have
\begin{equation}
     \|\eps_n(\Xi) \|_2 \stackrel{\textit{a.s.}}{\to}  0\,.
\end{equation} Finally, we use the continuity of $\eps_n(\Xi)$ and separability of $\bS^{p-1}$ to conclude the statement of the lemma. 
    
\end{proof}

\begin{lemma} \label{lemma:l2-convergence}
    Assume \ref{assmp:rep-convergence}. Then 
 as $n\to \infty$ the following convergence holds.  
    \begin{equation}
          \max_V \big|\bL_{\text{\emph{linear}}}(V) -  \widetilde {\bL}_{\text{\emph{linear}}}(V)\big|  \stackrel{\textit{a.s.}}{\to} 0\,.
    \end{equation}
\end{lemma}

\begin{proof}[Proof of the Lemma \ref{lemma:l2-convergence}]
We start with an expansion of $\bL_{\text{{linear}}}(V) -  \widetilde {\bL}_{\text{{linear}}}(V)$. 
\begin{equation}
    \begin{aligned}
        & \bL_{\text{{linear}}}(V) -  \widetilde {\bL}_{\text{{linear}}}(V)\\
        & = \textstyle\frac 1{n} 
    \sum_{(k_1, i_1)}
    \frac1{d(k_1, i_1)} 
    \sum_{(k_2, i_2)}
     e_{(k_1, i_1), (k_2, i_2)} \big\{\frac{v_{k_1, i_1}^\top v_{k_2, i_2}}{\tau}\big\}\\
     & ~~~~\textstyle  - \sum_{k_1} \hat \pi_{k_1} \frac{\sum_{k_2} \hat \pi_{k_2} \alpha_{k, k_2} }{\sum_{k' \in [K]} \hat \pi_{k'} \alpha_{k_1, k'}} \big\{\frac{h_{k_1}^\top h_{k_2}}{\tau}\big\}\\
     & = \textstyle\frac 1{n} 
    \sum_{(k_1, i_1)}
    \frac1{d(k_1, i_1)} 
    \sum_{(k_2, i_2)}
     e_{(k_1, i_1), (k_2, i_2)} \big\{\frac{v_{k_1, i_1}^\top v_{k_2, i_2}}{\tau}\big\}\\
     & ~~~~\textstyle  - \frac 1n\sum_{(k_1, i_1)}  \frac{\sum_{k_2} \hat \pi_{k_2} \alpha_{k, k_2} }{\sum_{k' \in [K]} \hat \pi_{k'} \alpha_{k_1, k'}} \big\{\frac{v_{(k_1, i_1)}^\top h_{k_2}}{\tau}\big\}\\
     & = \textstyle \frac 1{n\tau}\sum_{(k_1, i_1)} v_{(k_1, i_1)}^\top \big\{ \frac{\sum_{(k_2, i_2)}
     e_{(k_1, i_1), (k_2, i_2)}  v_{k_2, i_2}}{d(k_1, i_1)}  - \frac{\sum_{k_2} \hat \pi_{k_2} \alpha_{k_1, k_2} h_{k_2}}{\sum_{k' \in [K]} \hat \pi_{k'} \alpha_{k, k'}} \big\} \,.
    \end{aligned}
\end{equation} Following the expansion we define 
\begin{equation}
    \begin{aligned}
        \eps_1(V) &= \textstyle \frac 1n \sum_{(k_2, i_2)}
     e_{(k_1, i_1), (k_2, i_2)}  v_{k_2, i_2}  - \sum_{k_2} \hat \pi_{k_2} \alpha_{k_1, k_2} h_{k_2} \\
     \eps_2 & = \textstyle \frac{d(k_1,i_1)}{n} - \sum_{k' } \hat \pi_{k'} \alpha_{k_1, k'},
    \end{aligned}
\end{equation} and note that $\eps(V) = \eps_1(V)$ and $\eps(1) = \eps_2$, where $\eps(\cdot)$ is defined in lemma \ref{lemma:tech1-as}. As a conclusion of the lemma we have 
\begin{equation} \label{eq:c-24}
   \textstyle  \max_V \|\eps_1(V)\|_2 \stackrel{\textit{a.s.}}{\to} 0, ~\text{and} ~ |\eps_2| \stackrel{\textit{a.s.}}{\to}  0\,.
\end{equation}

 Next, we notice that 
\begin{equation}
    \begin{aligned}
        & \max_V\big|\bL_{\text{{linear}}}(V) -  \widetilde {\bL}_{\text{{linear}}}(V)\big| \\
        & \le \max_V\textstyle \frac 1{\tau}\sum_{(k_1, i_1)} \hat \pi_{k_1} \|h_{k_1}\|_2 \cdot \Big\| \frac{\sum_{(k_2, i_2)}
     e_{(k_1, i_1), (k_2, i_2)}  v_{k_2, i_2}}{d(k_1, i_1)}  - \frac{\sum_{k_2} \hat \pi_{k_2} \alpha_{k, k_2} h_{k_2}}{\sum_{k' \in [K]} \hat \pi_{k'} \alpha_{k, k'}} \Big\|_2 \\
      & \le \textstyle \frac 1{\tau}\sum_{(k_1, i_1)} \hat \pi_{k_1}  \cdot \max_V\Big\| \frac{\frac1n\sum_{(k_2, i_2)}
     e_{(k_1, i_1), (k_2, i_2)}  v_{k_2, i_2}}{\frac{d(k_1, i_1)}n}  - \frac{\sum_{k_2} \hat \pi_{k_2} \alpha_{k, k_2} h_{k_2}}{\sum_{k' \in [K]} \hat \pi_{k'} \alpha_{k, k'}} \Big\|_2 \,,
    \end{aligned}
\end{equation} where we focus on each term within the sum
\begin{equation}
    \begin{aligned}
        & \max_V\textstyle  \Big\| \frac{\frac1n\sum_{(k_2, i_2)}
     e_{(k_1, i_1), (k_2, i_2)}  v_{k_2, i_2}}{\frac{d(k_1, i_1)}n}  - \frac{\sum_{k_2} \hat \pi_{k_2} \alpha_{k_1, k_2} h_{k_2}}{\sum_{k' \in [K]} \hat \pi_{k'} \alpha_{k_1, k'}} \Big\|_2\\
     & = \max_V\textstyle  \Big\| \frac{\sum_{k_2} \hat \pi_{k_2} \alpha_{k_1, k_2} h_{k_2} + \eps_1(V)}{\sum_{k' \in [K]} \hat \pi_{k'} \alpha_{k_1, k'} + \eps_2}  - \frac{\sum_{k_2} \hat \pi_{k_2} \alpha_{k_1, k_2} h_{k_2}}{\sum_{k' \in [K]} \hat \pi_{k'} \alpha_{k_1, k'}} \Big\|_2\\
     & = \max_V\textstyle   \frac{\big\|\eps_2 \{\sum_{k_2} \hat \pi_{k_2} \alpha_{k_1, k_2} h_{k_2}\} + \eps_1(V) \{\sum_{k' \in [K]} \hat \pi_{k'} \alpha_{k_1, k'}\} \big\|_2}{\big|\{\sum_{k' \in [K]} \hat \pi_{k'} \alpha_{k_1, k'} + \eps_2\}\cdot \{\sum_{k' \in [K]} \hat \pi_{k'} \alpha_{k_1, k'}\}\big|}\\
     & \le \textstyle   \frac{|\eps_2| \cdot \max_V\big\| \sum_{k_2} \hat \pi_{k_2} \alpha_{k_1, k_2} h_{k_2}\big\|_2 + \big\{\max_V\|\eps_1(V)\|_2 \big\}\cdot  \big\{\sum_{k' \in [K]} \hat \pi_{k'} \alpha_{k_1, k'} \big\}}{\big|\sum_{k' \in [K]} \hat \pi_{k'} \alpha_{k_1, k'} + \eps_2\big|\cdot \big|\sum_{k' \in [K]} \hat \pi_{k'} \alpha_{k_1, k'}\big|}\,.
    \end{aligned}
\end{equation}
Here, both $|\eps_2|$ and  $\max_V\|\eps_1(V)\|_2$ convergence almost surely to zero. Since $\sum_{k' \in [K]}  \pi_{k'} \alpha_{k_1, k'} > 0$  both $\sum_{k' \in [K]} \hat \pi_{k'} \alpha_{k_1, k'} > 0$ and $\sum_{k' \in [K]} \hat \pi_{k'} \alpha_{k_1, k'} + \eps_2 > 0$ for sufficiently large $n$. Finally, 
\begin{equation}
   \max_V \textstyle \big\| \sum_{k_2} \hat \pi_{k_2} \alpha_{k_1, k_2} h_{k_2}\big\|_2 \le  \sum_{k_2} \hat \pi_{k_2} \alpha_{k_1, k_2}\big\{\max_V\| h_{k_2}\|_2 \big\} \le \sum_{k_2} \hat \pi_{k_2} \alpha_{k_1, k_2}\,.
\end{equation}
Thus we have 
\begin{equation}
    \max_V\textstyle  \Big\| \frac{\frac1n\sum_{(k_2, i_2)}
     e_{(k_1, i_1), (k_2, i_2)}  v_{k_2, i_2}}{\frac{d(k_1, i_1)}n}  - \frac{\sum_{k_2} \hat \pi_{k_2} \alpha_{k_1, k_2} h_{k_2}}{\sum_{k' \in [K]} \hat \pi_{k'} \alpha_{k_1, k'}} \Big\|_2 \stackrel{\textit{a.s.}}{\to} 0
\end{equation} which proves the lemma that
\begin{equation}
    \max_V \big|\bL_{\text{{linear}}}(V) -  \widetilde {\bL}_{\text{{linear}}}(V)\big|  \stackrel{\textit{a.s.}}{\to} 0
\end{equation}

\end{proof}




\section{Supplement for SBM simulation in Section \ref{sec:SMB-simulation}}
\label{sec:supp_sbm}
We generate an SBM dataset using \texttt{graspologic.simulations.sbm} function. We obtain the representation by optimizing the node2vec loss in eq. \eqref{eq:node2vec} using a gradient descend algorithm that has step size for $t$-th step as $\text{lr}_t = 0.1 \times {t}^{-0.2}$  and $T = 30000$ as the number of iterations. 

\end{document}